\documentclass[preprint,12pt]{elsarticle}

\usepackage[utf8]{inputenc}

\usepackage[title]{appendix}
\usepackage{float}

\usepackage{amsmath,amsfonts,amsthm,amssymb,mathtools,bbm,bm}
\usepackage{mathtools}
\usepackage{mathrsfs}
\usepackage{amsthm}
\newcommand{\real}{\mathbbm{R}}
\newcommand{\interval}{\mathbbm{I}\mathbbm{R}}

\newtheorem{definition}{Definition}

\newtheorem{theorem}{Theorem}

\newtheorem{corollary}{Corollary}
\newtheorem{problem}{Problem}

\usepackage{graphicx}
\usepackage{caption}
\usepackage{subcaption}
\usepackage{tikz}
\usepackage[mode=buildnew]{standalone}

\usepackage{multirow}

\usepackage{algorithm}
\usepackage{algpseudocode}
\algrenewcommand\algorithmicrequire{\textbf{Input:}}
\algrenewcommand\algorithmicensure{\textbf{Output:}}

\usepackage{url}
\usepackage{xr-hyper}

\usepackage{hyperref}
\usepackage{footmisc}

\usepackage{comment}

\usepackage{siunitx}

\journal{Neural Networks}

\makeatletter
\renewcommand*\env@matrix[1][*\c@MaxMatrixCols c]{%
	\hskip -\arraycolsep
	\let\@ifnextchar\new@ifnextchar
	\array{#1}}
	\makeatother

\begin{document}

\begin{frontmatter}

\title{Towards Global Neural Network Abstractions with Locally-Exact Reconstruction}

\author[UoM]{Edoardo Manino\corref{CorAuth}}
\ead{edoardo.manino@manchester.ac.uk}

\author[UFAM]{Iury~Bessa}
\ead{iurybessa@ufam.edu.br}

\author[UoM]{Lucas Cordeiro}
\ead{lucas.cordeiro@manchester.ac.uk
}

\address[UoM]{University of Manchester, Department of Computer Science,\\ Oxford Road, Manchester M13 9PL, United Kingdom}

\address[UFAM]{Federal University of Amazonas, Department of Electricity,\\ Avenida General Rodrigo Octavio Jordão Ramos, 1200 - Coroado I,\\ Manaus - AM, 69067-005, Brazil}

\cortext[CorAuth]{Corresponding author}

\begin{abstract}
Neural networks are a powerful class of non-linear functions. However, their black-box nature makes it difficult to explain their behaviour and certify their safety. Abstraction techniques address this challenge by transforming the neural network into a simpler, over-approximated function. Unfortunately, existing abstraction techniques are slack, which limits their applicability to small local regions of the input domain. In this paper, we propose Global Interval Neural Network Abstractions with Center-Exact Reconstruction (GINNACER). Our novel abstraction technique produces sound over-approximation bounds over the whole input domain while guaranteeing exact reconstructions for any given local input. Our experiments show that GINNACER is several orders of magnitude tighter than state-of-the-art global abstraction techniques, while being competitive with local ones.
\end{abstract}


\end{frontmatter}

\section{Introduction}

\sloppy

Neural networks have been applied to a large number of predictive tasks. The main reason for their versatility is their ability to approximate a large class of nonlinear functions~\citep{Hornik1989,Sonoda2017}. However, at the same time, trained neural networks tend to have a black-box nature, which limits our ability to explain their behaviour~\citep{Adadi2018}. More worryingly, even when we have a priori formal expectations on their behaviour, e.g., robustness to input noise, the act of certifying whether a neural network satisfies our requirements comes at a very high computational cost~\citep{Katz2017,Wang2022b}.

A common strategy in dealing with the complexity of analysing neural networks is transforming them into equivalent models described by better-behaved classes of functions. Examples from theoretical works include max-affine splines~\citep{Balestriero2021}, residual networks with one neuron per layer~\citep{Lin2018}, prototypical ``sawtooth'' rectified linear unit (ReLU) networks~\citep{Telgarsky2015,Yarotsky2017} and Radon or ridgelet integral representations~\citep{Carroll1989,Petrosyan2020}. However, the purpose of these transformations is to prove properties over whole classes of network architectures, e.g., the existence of approximating networks with given error bounds, rather than providing a constructive algorithm that outputs such a network.

In stark contrast with these works, there exist a variety of approximate techniques that can compress the size of a given neural network. Pruning techniques can remove network parameters whose contribution is negligible~\citep{Liang2021,Gunnet2018}. Distillation techniques train a simpler network to approximate the output of the hidden layers of the original network~\citep{Wang2022}. Quantisation techniques reduce the accuracy of arithmetic operations to take advantage of hardware acceleration~\citep{Hubara2017}. Frequently, these techniques are combined to minimise the computational cost of executing a neural network in terms of power, latency or memory~\citep{Han2015,Polino2018}. However, they offer no formal guarantees on the approximation errors introduced during the compression process.

Ideally, we would like to compress any given neural network while keeping formal guarantees on the original model. A recent development in such direction is lossless compression~\citep{Serra2020,Serra2021,Sourek2021}. The algorithms in this family can identify which ReLU activation functions in the network are always active or inactive and remove them accordingly. Unfortunately, lossless algorithms can only produce equivalent networks; thus, their compression scope is somewhat limited~\citep{Kumar2019}.

More promising techniques can be grouped under the umbrella term of \textit{abstractions}~\citep{Gehr2018,Mirman2018}. An abstraction is a set-valued mapping, i.e., a function whose output is a set instead of a point. In particular, this output set is guaranteed to be an over-approximation of the original and more complex function. More formally, for a given abstraction's domain of validity, the image of the original function is always within the abstraction's output set~\citep{Cousot1992}. Depending on that domain of validity, we discriminate between \textit{local} and \textit{global} abstractions.

On the one hand, local abstractions are valid for a small region of the input domain only~\citep{Weng2018,Zhang2018,Salman2019,Singh2019}. As such, they are suitable for verifying the robustness of neural networks to local perturbations. Typically, their design is neuron-centric: each activation function is over-approximated by bounding its output between a set of linear~\citep{Paulsen2022} or quadratic constraints~\citep{Zhang2018}. Still, these abstractions' main goal is not to reduce the network complexity but to encode it into a set of constraints that a solver can handle.

On the other hand, global abstractions are valid for any input~\citep{Prabhakar2019,Elboher2020}. In this regard, Prabhakar and Afzal~\citep{Prabhakar2019} propose to transform the original neural network so that it can operate in interval arithmetic \citep{Jaulin2001} and maintain an upper and lower bound on the value of each neuron. Furthermore, by over-approximating the individual intervals, the authors can merge similar neurons and reduce the overall complexity of the global abstraction function. Similarly, Elboher \textit{et al.}~\citep{Elboher2020} introduce a global abstraction algorithm that creates multiple copies of each neuron, and then merge them together. The advantage of their algorithm is that the abstraction function is itself another neural network, which grants a high degree of integration with existing verification tools~\citep{Katz2019}.

A significant drawback of local and global abstraction techniques is their slackness~\citep{Bak2020,Elboher2020}. In general, there is a trade-off between the complexity of the abstracted neural network and the size of the over-approximation margin it produces. In practice, even abstracting a few neurons may generate large output sets that become meaningless. This phenomenon is not surprising: established techniques from control theory, like Taylor approximations~\citep{Ivanov2021} and Bernstein polynomials~\citep{Huang2019}, are known to suffer from the same over-approximation explosion when applied to highly non-linear functions.

This paper proposes an algorithm to compute Global Interval Neural Network Abstractions with Center-Exact Reconstruction (GINNACER). Our approach to the complexity-approximation trade-off is novel: we focus on guaranteeing exact reconstruction for a given input centroid while keeping sound over-approximation bounds elsewhere. This way, we obtain tighter global abstractions in a specific local region of the input. Empirically, our GINNACER abstractions are orders of magnitude tighter than state-of-the-art global abstractions while on par with local abstractions.

More specifically, our contributions are the following:
\begin{itemize}
    \item We introduce the Inactive Canonical Form (ICF), an exact transformation of ReLU networks where all activation functions are inactive for a specific input centroid.
    \item We introduce a layer-wise abstraction that can handle interval inputs and outputs and uses fewer ReLU activations than the original layer.
    \item We introduce the GINNACER algorithm, which computes an individual abstraction for each layer of the network. The abstracted layers are aligned so that their output intervals have size zero for a specific input centroid.
    \item We compare GINNACER with state-of-the-art global and local abstractions. Our comparison is not dependent on a downstream task (e.g. robustness verification) like in previous studies but measures the quality of the abstraction in isolation.
\end{itemize}

The rest of this paper is organized as follows. In Section \ref{sec:related_work} we clarify the differences between GINNACER and the existing literature. In Section \ref{sec:preliminaries} we introduce a motivating example and formally state the GINNACER requirements. In Section \ref{sec:algorithm}, we present our GINNACER algorithm and provide proof of its theoretical properties. In Section \ref{sec:experiments}, we compare GINNACER against the state-of-the-art and study its empirical properties. Finally, in Section \ref{sec:discussion}, we discuss the potential uses of GINNACER, the connections with related research fields and a roadmap for future work. The data, neural networks and code of our experiments is freely available at \citep{Manino2022}.

\section{Related Work}
\label{sec:related_work}

Due to the functional complexity of modern neural networks, there have been many efforts to cast them to a simpler class of functions. The purpose and application of these neural transformation methods are very different in nature. Here, we propose a high-level taxonomy that focuses on the mathematical relationship between the original neural network $f$ and its transformed version $g$.

First, we have neural network approximations, where $f(x)\approx g(x)$. In this category, neural network compression methods are the most prominent~\citep{Han2015,Polino2018}. Their goal is to reduce the computational requirements of the neural network at inference time. These methods include pruning individual weights or neurons~\citep{Liang2021,Gunnet2018}, executing the neural network in integer arithmetic or any similar quantized representation~\citep{Hubara2017}, or training a smaller network to mimic the hidden representations of the original one~\citep{Wang2022}. A more recent development in the approximation category comes from privacy-preserving machine learning, where it has been proposed to approximate the non-linear activation functions of a neural network with high-degree polynomials~\citep{Obla2020,Lee2022}. This approximation allows inference on private inputs via homomorphic encryption. In general, neural network approximations give no guarantees on the approximation error and are usually evaluated in terms of loss of predictive accuracy concerning the original network.

Second, we have equivalent transformations, where $f(x)=g(x),\forall x$. On the one hand, some of these efforts are theoretical in nature, and their purpose is proving some general mathematical properties of neural networks~\citep{Balestriero2021}. On the other hand, there have been efforts towards constructing lossless compression algorithms for neural networks~\citep{Serra2020,Serra2021,Sourek2021}. These algorithms can identify the neurons whose activation function state is either active or inactive for all possible inputs $x$. Then, these neurons can be safely replaced with a linear function or removed altogether. However, only a minority of the neurons in the original network can be processed in this way.

Third, we have global functional abstractions, where $\underline{g}(x)\leq f(x)\leq\overline{g}(x),\forall x$. The main challenge is maintaining a valid upper and lower bound for any $x$ in the input domain. Prabhakar \textit{et al.} \citep{Prabhakar2019} achieve this by allowing the entire neural network to operate in interval arithmetic. Instead, Elboher \textit{et al.} \citep{Elboher2020} propose to keep explicit copies of each neuron that carry the upper and lower bound information. Both methods can merge neurons with similar weights, thus yielding a pair of functions $\underline{g},\overline{g}$ that contains fewer non-linearities. However, the functions $\underline{g},\overline{g}$ produced by the aforementioned methods tend to be quite slack, as we show in our experiments (see Sections~\ref{sec:motivating_example} and \ref{sec:empirical_comparison}). In contrast, GINNACER yields tighter abstractions while maintaining global soundness.

Fourth, we have local functional abstractions, where $\underline{g}(x)\leq f(x)\leq\overline{g}(x)$ for a small sub-domain $x\in\mathcal{D}$. These have the advantage that they only need to represent a small portion of the original function $f$. Traditionally, non-linear systems have been analysed with Taylor models~\citep{Makino2003}, which give a local Taylor polynomial approximation and a corresponding worst error interval. This corpus of techniques has been applied to neural networks with a fair degree of success~\citep{Ivanov2021,Huang2022}. In contrast, more recent advances in local functional abstractions for neural networks have been driven by linear techniques. For example, Weng \textit{et al.} \citep{Weng2018} propose FastLin, which computes parallel bounds for each non-linear activation function in the network. Similarly, Zhang \textit{et al.} \citep{Zhang2018} propose CROWN, where the linear bounds are optimized to reduce the over-approximation margin. Crucially, this line of work can be generalised to arbitrary computation graphs, as demonstrated in~\citep{Xu2020}. When paired with branch-and-bound techniques, linear local functional abstractions represent the state-of-the-art in robustness verification of neural networks~\citep{Muller2022}.

Fifth, we have set-based abstractions, where $f(x)\in g(\mathcal{D})$ for a small sub-domain $x\in\mathcal{D}$, and the input-output relation for each specific $x\in\mathcal{D}$ is lost. The main challenge in this category is defining an efficient and flexible  representation of a set. In this regard, \citep{Singh2018} uses linear zonotopes, \citep{Kochdumper2021} uses sparse polynomial zonotopes, and \citep{Tran2019,Bak2021} uses star sets. Alternatively, any arbitrary convex relaxation of the non-linear activation functions can be employed \citep{Salman2019}. For instance, DeepPoly collects linear constraints for each neuron while executing the neural network symbolically and computes the output set $g(\mathcal{D})$ by solving a linear programming problem~\citep{Singh2019b}. Similarly, the PRIMA method represents the convex hull of multiple neurons at the same time via linear constraints, thus reducing the over-approximation margin \citep{Muller2022b}.

\section{Preliminaries}
\label{sec:preliminaries}

Before introducing our contribution, let us clarify our research aims with a motivating example and a formal problem statement.

\subsection{Motivating Example}
\label{sec:motivating_example}

For each input, a neural network abstraction produces a set that is guaranteed to contain the output of the original network. To see what that looks like in practice, we train a neural network and compare its output with our GINNACER abstractions and several state-of-the-art alternatives.

More specifically, we train a fully-connected feedforward network $f(x)$ with two hidden layers of $64$ neurons each and ReLU activation functions. We construct the dataset by uniformly sampling the following polynomial in the range $x\in[-10,10]$ with a sampling step of $0.1$ ($201$ samples in total):
\begin{equation}
\label{eq:pade_jacoby_poly}
    y=\frac{- 0.035x^5 - 0.12x^3 + x  }{0.021x^6 - 0.10x^4 + 0.55x^2 + 1}
\end{equation}
the result is depicted in Figure \ref{fig:abs_comp} (black line). Note that Subfigures \ref{fig:global_comp} and \ref{fig:local_comp} present the same neural network $f(x)$ with different rescaling of the $y$ axis.

\begin{figure}[t]
\centering
    \begin{subfigure}[b]{0.49\textwidth}
    \centering
        \includegraphics[width=\textwidth]{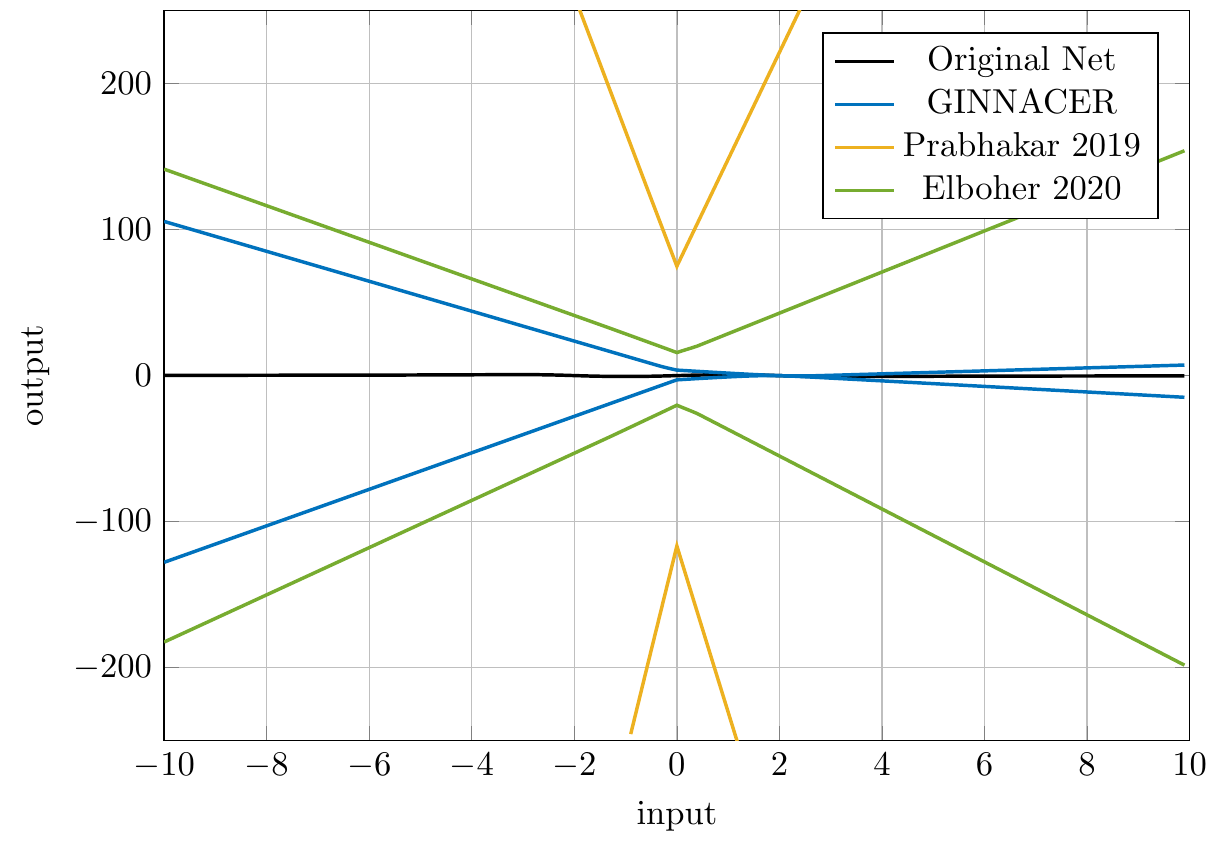}
        \caption{Global abstractions (zoomed out)}
        \label{fig:global_comp}
    \end{subfigure}
    \begin{subfigure}[b]{0.49\textwidth}
    \centering
        \includegraphics[width=\textwidth]{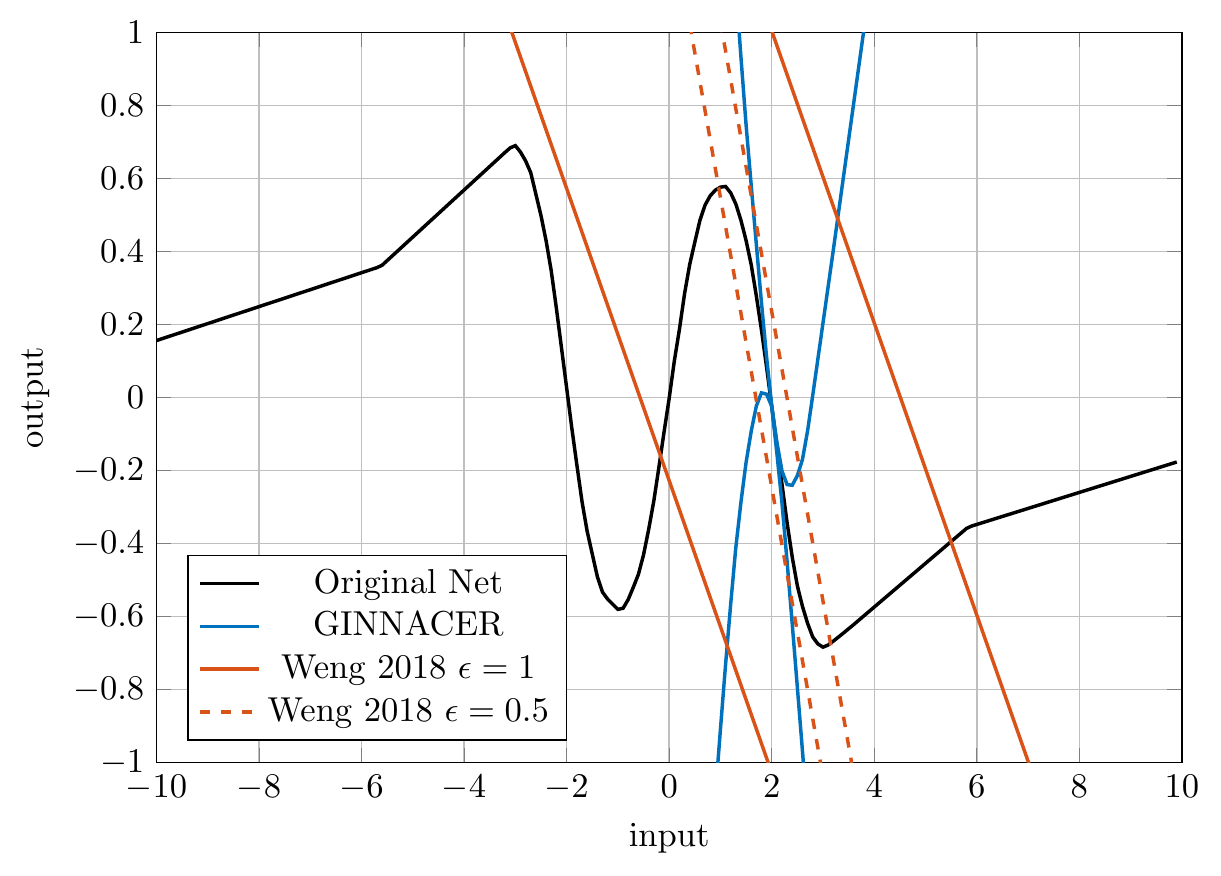}
        \caption{Local abstractions (zoomed in)}
        \label{fig:local_comp}
    \end{subfigure}
\caption{Comparison between global and local abstraction methods.}
\label{fig:abs_comp}
\end{figure}

State-of-the-art \textit{global} abstraction techniques \cite{Prabhakar2019,Elboher2020} produce upper and lower bounds on the output of $f(x)$ for all $x$ (more details on these algorithms in Section~\ref{sec:list_of_sota}), but introduce a large over-approximation margin in doing so (see Subfigure~\ref{fig:global_comp}). Conversely, the bounds of \textit{local} abstraction techniques \cite{Weng2018} tend to be tighter, but are valid for a small input range $x\in[c-\epsilon,c+\epsilon]$ only (in Subfigure \ref{fig:local_comp} we have $c=2$).

The proposed GINNACER algorithm described in Section~\ref{sec:algorithm} combines the best of both worlds: global abstraction guarantees and exact reconstruction around a centre point $c$. Furthermore, the GINNACER algorithm reduces the network's complexity, which we measure as the number of ReLU activation functions used in the abstracted network. More specifically, all global abstraction in Figure \ref{fig:abs_comp} use only 9 out of 64 ReLUs for the first layer and 55 out of 64 for the second.\footnote{The local abstraction by Weng \textit{et al.}~\citep{Weng2018} is linear (see Section \ref{sec:list_of_sota} for details).}

\subsection{Definitions and Notation}
\label{sec:definitions}

\paragraph{Notation.} $\mathbb{R}^{d}$ denotes the $d$-dimensional Euclidean space, and $\mathbb{R}^{d_1 \times d_2}$ is the set of all $d_1\times d_2$  real matrices. For two vectors $x_1, x_2 \in \mathbb{R}^d$ or matrices $A_1, A_2 \in \mathbb{R}^{d_1\times d_2}$, the relations $x_1 \leq x_2$ and $A_1 \leq A_2$ are understood elementwise.  $\interval^d$ denotes the cartesian product of the $d$ sets of closed intervals, i.e.,  $\interval\equiv\{[\underline{z},\overline{z}]:\underline{z},\overline{z}\in\real^{d}\land \underline{z}\leq \overline{z}\}$. For a matrix $A\in \mathbb{R}^{d_1\times d_2}$, we define $A^{+}=\max{\left\lbrace A, 0 \right\rbrace}$ and $A^{-}=\min{\left\lbrace A, 0 \right\rbrace}$ element-wise. Furthermore, we represent vectors as columns.

Let us now formalize the concepts presented in our motivating example.

\begin{definition}[ReLU Activation Function]
    The Rectified Linear Unit (ReLU) is a piece-wise linear activation function $\sigma:\real\to\real^+$, where $\sigma(z)=\max(z,0)$.
\end{definition}

\begin{definition}[Fully-Connected Layer]
    Let $y^{\ell}\in\real^{n^{\ell}}$ and $x^{\ell}\in\real^{n^{\ell}}$ be the pre-activation and post-activation vectors at layer $\ell$, where for each neuron $i\in[1,n^{\ell}]$ we have $x^{\ell}_i=\sigma(y^{\ell}_i)$, with activation function $\sigma$. The pre-activation values depend on the previous layer via a linear mapping $y^{\ell}=W^{\ell}x^{\ell-1}+b^{\ell}$, with weight matrix $W\in\real^{n^{\ell}\times n^{\ell-1}}$ and bias vector $b^{\ell}\in\real^{n^{\ell}}$.
\end{definition}

\begin{definition}[Feedforward Neural Network]
    A feedforward neural network $f(x^0)$ is the composition of $m$ fully-connected (hidden) layers and a final linear output layer:
    \begin{equation}
    \label{eq:feedforward_nn}
        y^m=f(x^0)=W^m\sigma(W^{m-1}\sigma(\dots\sigma(W^1x^0+b^1))+b^{m-1})+b^m
    \end{equation}
    where $x^0\in\real^{n^0}$ is the network input vector, $y^m\in\real^{n^m}$ the output vector, and we use $\sigma(v)$ as an element-wise operator on vector $v$, with some abuse of notation.
\end{definition}

\begin{definition}[Interval Abstraction]
\label{def:interval_abstraction}
    Let $f:\real^{d_1}\to\real^{d_2}$ be a vector-valued function defined on $d_1$ real inputs and $d_2$ real outputs. Let $g:\real^{d_1}\to\interval^{d_2}$ be a set-valued function that maps $d_1$ real inputs to closed intervals of $\interval^{d_2}$. We say $g$ is an interval abstraction of $f$ if $f(x)\in g(x)$ for all $x\in\mathscr{D}\subseteq \real^{d_1}$, where $\mathscr{D}$ is the domain of validity of the abstraction. The abstraction $g$ is said to be global if $\mathscr{D}= \real^{d_1}$ 
\end{definition}

\subsection{Problem Statement}
\label{sec:problem_statement}

The challenge of computing effective neural network abstractions lies in balancing two opposite requirements: simplicity of the abstracted function and minimal size of its output set. To solve this problem we propose the GINNACER for representing neural networks. The definition of GINNACER is given as follows.

\begin{definition}[GINNACER: Global Interval Neural Network Abstraction with Center-Exact Reconstruction]
\label{def:ginnacer}
    Given a neural network $f:\real^{n^0}\to\real^{n^m}$ with $a_f$ non-linear activation functions, we say that $g:\real^{n^0}\to\interval^{n^m}$ is the GINNACER of $f$ if it satisfies the following properties:
    \begin{enumerate}
        \item \textbf{Global Soundness.} $f(x)\in g(x)$ for any $x\in\real^{n^0}$.
        \item \textbf{Reduced Non-Linearity.} $g$ is a neural network itself with $a_g\leq a_f$ non-linear activation functions.
        \item \textbf{Center-Exact Reconstruction.} There exist an input $x_c^0\in\real^{n^0}$ such that $g(x_c^0)=\{f(x_c^0)\}$, i.e. the output set contains only $f(x_c^0)$.
    \end{enumerate}
\end{definition}

Therefore, this paper aims to solve the following problem.

\begin{problem}
    Given a neural network $f:\real^{n^0}\to\real^{n^m}$ with $a_f$ non-linear activation functions, find a GINNACER of $f$ around input $x^0_c\in\real^{n^0}$.
\end{problem}

In Section \ref{sec:algorithm}, we give an algorithm to compute the GINNACER of a neural network $f$ when $f$ is a fully-connected feedforward neural network with ReLU activation functions. Among the existing literature, the algorithm in \cite{Elboher2020} is also restricted to this setting. We plan to address different network architectures in our future work.

\section{The GINNACER Algorithm}
\label{sec:algorithm}

Here, we introduce an algorithm that satisfies the GINNACER requirements of Definition~\ref{def:ginnacer}. Our algorithm computes interval abstractions $g^{\ell}:\interval^{n^{\ell-1}}\to\interval^{n^{\ell}}$ for all layer of the original neural network $f$, and composes them into a full network abstraction:
\begin{equation}
\label{eq:abstract_composition}
    g^{\mathrm{full}}
    \equiv g^1\circ g^2 \circ\dots\circ g^{m-1}\circ g^m
\end{equation}
where plugging degenerate intervals as the input, i.e. $g([x,x])$, recovers the notation in Definition \ref{def:interval_abstraction}.

In this respect, our GINNACER algorithm maintains a sound interval over the value of each post-activation vector, i.e. $x_i^{\ell}\in[\underline{x}_i^{\ell},\overline{x}_i^{\ell}]$. As such, we choose to present our algorithm as a sequence of transformations from the original neural layers $x^{\ell}=\sigma(W^{\ell}x^{\ell-1}+b^{\ell})$ to their abstractions $g^{\ell}$. We address the challenge of composing different abstracted layers together in Section~\ref{sec:full_abstraction}, and we give the pseudocode of the full GINNACER algorithm in Section \ref{sec:algo_code}.

\subsection{Representing Negative Inputs}
\label{sec:neg_input_trick}

Later in Section~\ref{sec:clustering}, we need the assumption that the inputs $x^{\ell-1}$ of all layers $\ell$ are non-negative. Such assumption is clearly satisfied for all layers where $x^{\ell-1}_i=\sigma(y^{\ell-1}_i)$ and $\sigma$ is the ReLU function. At the same time, the inputs $x^0$ of the first hidden layer $\ell=1$ can take any value in general, including negative ones. In order to apply our abstraction technique to the first layer of the neural network, we propose a lossless transformation of the input layer as follows.

\begin{definition}[Negative Input Pre-Layers]
\label{def:pre_layers}
    Let $\ell=0$ be an additional layer with the following weights and biases:
    \begin{equation}
        W^0_{\mathrm{pre}}
        =\begin{bmatrix}I\\-I\end{bmatrix}
        \in\real^{2n^0\times n^0}
        \qquad\text{and}\qquad
        b^0_{\mathrm{pre}}=\begin{bmatrix}0\\\vdots\\0\end{bmatrix}\in\real^{2n^0}
    \end{equation}
    where $n^0$ is the dimension of the neural network inputs. Furthermore, assume that the weights $W^1$ and biases $b^1$ of layer $\ell=1$ are modified as follows:
    \begin{equation}
        W^1_{\mathrm{pre}}
        =\begin{bmatrix}W^1&-W^1\end{bmatrix}
        \in\real^{n^1\times 2n^0}
        \qquad\text{and}\qquad
        b^1_{\mathrm{pre}}=b^1
    \end{equation}
    We define the post-activation values of these two pre-layers as follows:
    \begin{equation}
        x^0_{\mathrm{pre}}=\sigma(W^0_{\mathrm{pre}}x^{-1}+b^0_{\mathrm{pre}})
        \qquad\text{and}\qquad
        x^1_{\mathrm{pre}}=\sigma(W^1_{\mathrm{pre}}x^0_{\mathrm{pre}}+b^1_{\mathrm{pre}})
    \end{equation}
    where the input to the modified pre-layers $x^{-1}$ has the same dimensionality as the original network input $x^{0}$.
\end{definition}

\begin{theorem}
    The post-activation vector $x^1_{\mathrm{pre}}$ of the pre-layers is identical to the original post-activations $x^1$, when the respective inputs $x^{-1}=x^0$ are identical.
\end{theorem}

\begin{proof}
    We can decompose the original input as $x^0=\sigma(x^0)-\sigma(-x^0)$, which yields $y^1=W^1\sigma(x^0)-W^1\sigma(-x^0)+b^1$. Now, recall that $y^1_{\mathrm{pre}}=W^1\sigma(Ix^{-1})-W^1\sigma(-Ix^{-1})+b^1$ according to Definition \ref{def:pre_layers}. These yield the result in the theorem when $x^{-1}=x^0$.
\end{proof}

From the architectural perspective, the transformation in Definition \ref{def:pre_layers} generates an additional ReLU layer that is prepended to the neural network. While this layer is necessary to make GINNACER satisfy the global soundness requirements from Definition \ref{def:ginnacer}, its presence might increase the number of non-linearities to unacceptable levels. This is especially true for neural networks with a large number of inputs (e.g. image classification networks). Furthermore, if the neural network has naturally bounded inputs, it might be possible to rescale them to take positive values only. In such cases, it is always possible to keep layer $\ell=1$ as is, and apply GINNACER from layer $\ell=2$ onward (see also \ref{sec:tradeoff}).

However, for the sake of clarity, in the remainder of this paper we assume that the lossless transformation in Definition \ref{def:pre_layers} is always applied. As such, we assume that $x^{-1}$ is always the transformed network input and omit the subscript $pre$ from the weights, biases and activations of layers $\ell\in\{0,1\}$.

\subsection{Converting to Inactive Canonical Form}
\label{sec:inactive_form}

Next, we introduce the \textit{Inactive Canonical Form} (ICF) of a fully-connected ReLU layer. This is an equivalent layer whose ReLU activation functions are all inactive for a specific layer input $x_c^{\ell-1}$, which we call the \textit{centroid}. Later in Section~\ref{sec:clustering}, the ICF will allow us to satisfy the center-exact reconstruction requirement of GINNACER (see Definition~\ref{def:ginnacer}).

\begin{definition}[Centroid Activation Matrices]
\label{def:centact_matrix}
    Let $A^{\ell}$ and $S^{\ell}$ be diagonal matrices of size $n^{\ell}\times n^{\ell}$, with the following entries:
    \begin{align}
        A^{\ell}_{ii} &=
        \begin{cases}
        1\qquad\text{if }(W^{\ell}x_c^{\ell-1}+b^{\ell})_i\geq0\\
        0\qquad\text{otherwise}
        \end{cases}\\
        S^{\ell}_{ii} &= 1 - 2A^{\ell}_{ii}
    \end{align}
    where $x_c^{\ell-1}$ is a specific input of layer $\ell$.
\end{definition}

The diagonal matrices $A^{\ell}$ and $S^{\ell}$ contain information about the activation pattern of the ReLUs when the layer input is $x_c^{\ell-1}$. Thanks to them, we can define the following equivalent layer transformation.

\begin{definition}[Inactive Canonical Form]
\label{def:icf}
    Given a ReLU layer with weights $W^{\ell}$ and biases $b^{\ell}$, define the following quantities:
    \begin{align}
    \label{eq:icf_r}
        r^{\ell}&=S^{\ell}\big(W^{\ell}x^{\ell-1}+b^{\ell}\big)\\
    \label{eq:icf_t}
        t^{\ell}&=A^{\ell}\big(W^{\ell}x^{\ell-1}+b^{\ell}\big)
    \end{align}
    The inactive canonical form of layer $\ell$ is:
    \begin{equation}
    \label{eq:icf_x}
        x^{\ell}_{ICF} = \sigma(r^{\ell}) + t^{\ell}
    \end{equation}
\end{definition}

\begin{theorem}
\label{th:icf_equivalence}
    The post-activation vector $x^{\ell}_{ICF}$ in Definition \ref{def:icf} is identical to $x^{\ell}=\sigma(W^{\ell}x^{\ell-1}+b^{\ell})$ for any $W^{\ell}$, $b^{\ell}$ and $x^{\ell-1}$.
\end{theorem}

\begin{proof}
    If neuron $i$ is already inactive at the centroid $x_c^{\ell-1}$, then it is left unchanged: $A_{ii}^{\ell}=0$ and $S_{ii}^{\ell}=1$ yield $r^{\ell}_i = y^{\ell}_i$ and $t^{\ell}_i=0$, thus $(x^{\ell}_{ICF})_i= \sigma(y^{\ell}_i)=x^{\ell}_i$. Else, any active neuron $j$ is associated with $A_{jj}^{\ell}=1$ and $S_{jj}^{\ell}=-1$, which yields $r^{\ell}_j = -y^{\ell}_j$ and $t^{\ell}_j=y^{\ell}_j$, thus $(x^{\ell}_{ICF})_j= \sigma(-y^{\ell}_j)+y^{\ell}_j$. Because of the relationship $\sigma(z)=\sigma(-z)+z$, we have $(x^{\ell}_{ICF})_j=x^{\ell}_j$.
\end{proof}

\subsection{Reducing the Number of Non-Linearities}
\label{sec:clustering}

Here, we introduce our primary abstraction step, where we replace each neuron's post-activation value $x_i^{\ell}$ with appropriate upper and lower bounds. Per the GINNACER requirements from Definition \ref{def:ginnacer}, we aim at reducing the number of ReLU activation functions while keeping the over-approximation margin zero for the centroid $x_c^{\ell-1}$. Our lower bound is straightforward:
\begin{definition}[Lower Bound]
\label{def:lower_bound}
    Given that the ReLU activation function satisfies $\sigma(r^{\ell}_i)\geq0$ for any $r^{\ell}_i$, let us bound Equation \ref{eq:icf_x} from below:
    \begin{equation}
    \label{eq:icf_x_lower}
        x^{\ell}=x^{\ell}_{ICF}\geq t^{\ell}
    \end{equation}
    with equality for any element $i$ such that $r^{\ell}_i\leq0$, e.g. when $x^{\ell-1}=x^{\ell-1}_c$ as per the ICF in Definition \ref{def:icf}.
\end{definition}

In contrast, our upper bound requires operating on multiple neurons at the same time:

\begin{definition}[Upper Bound]
\label{def:upper_bound}
    Let $D$ be a subset of neurons in layer $\ell$, and assume that we have $x_j^{\ell-1}\geq0$ for each layer input $j$. Then, define the maximum of the weights and biases across set $D$ as follows:\footnote{The idea of taking the maximum over the neuron weights and biases is fundamentally equivalent to the merging operations in \cite{Prabhakar2019} and \cite{Elboher2020}.}
    \begin{equation}
    \label{eq:cluster_weights}
        (V^{\ell}_D)_j
        \equiv\max_{i\in D}(S^{\ell}W^{\ell})_{ij}
        \qquad\text{and}\qquad
        u^{\ell}_D
        \equiv\max_{i\in D}(S^{\ell}b^{\ell})_i
    \end{equation}
    By construction, we can use $V^{\ell}_D$ and $u^{\ell}_D$ to bound Equation \ref{eq:icf_r} from above:
    \begin{equation}
    \label{eq:icf_r_upper}
        r^{\ell}_i\leq\hat{r}_D^{\ell}
        \equiv(V^{\ell}_D)^Tx^{\ell-1}+u^{\ell}_D
        \qquad\text{for all }i\in D
    \end{equation}
    Finally, Equation \ref{eq:icf_r_upper} yields the following upper bound:
    \begin{equation}
    \label{eq:icf_x_upper}
        x^{\ell}_i=(x^{\ell}_{ICF})_i
        \leq \sigma(\hat{r}_D^{\ell}) + t^{\ell}_i
        \qquad\text{for all }i\in D
    \end{equation}
\end{definition}

The upper bound in Definition~\ref{def:upper_bound} can be constructed from any subset $D$ of neurons. However, we would like the corresponding bound to produce no over-approximation at the centroid. This requirement can be formalized as follows.

\begin{definition}[Valid Neuron Subset]
\label{def:valid_subset}
    We call $D$ a valid subset of neurons for layer $\ell$ (in the GINNACER sense) if the associated upper bound from Equation \ref{eq:icf_r_upper} is negative at the centroid $x^{\ell-1}_c$. That is, $\hat{r}_D^{\ell}\equiv(V^{\ell}_D)^Tx^{\ell-1}_c+u^{\ell}_D\leq0$.
\end{definition}

Note that the ReLU activations of the neurons in a valid subset $D$ are replaced with a single ReLU $\sigma(\hat{r}_D^{\ell})$ as per Equation \ref{eq:icf_x_upper}. Thus, our objective is to partition the neurons from layer $\ell$ in the fewest number of subsets.

\begin{definition}[Valid Neuron Partition]
\label{def:valid_partition}
    Given the set of neurons $E^{\ell}\equiv\{1,\dots,n^{\ell}\}$ at layer $\ell$, we call $E^{\ell}=D_1\cup\dots\cup D_h$ a valid partition of $E^{\ell}$ (in the GINNACER sense) if $D_i\cap D_j=\emptyset$ for all $i\neq j$ and each $D_i\neq\emptyset$ is a valid subset according to Definition \ref{def:valid_subset}.
\end{definition}

At the same time, computing the best neuron subset partition is infeasible in practice. In fact, even the greedy algorithm used in \cite{Elboher2020} becomes computationally inefficient as the number of neurons per layer grows large, as we show in our experiments of Section \ref{sec:time_comparison}. For this reason, we settle on a light-weight neuron partition strategy in our implementation of GINNACER. We give a complete description of our algorithm in \ref{sec:partioning}.

\subsection{Propagating Interval Bounds}
\label{sec:interval_prop}

In Section \ref{sec:clustering} we compute upper and lower bounds on the output of layer $\ell$. As a consequence, any following layer $\ell'>\ell$ must be able to operate with interval inputs. To achieve this, let us first rewrite the upper bound in matrix form:

\begin{definition}[Layer-Wise Upper Bound]
\label{def:layer_upper_bound}
    Given a valid partition $E^{\ell}=D_1\cup\dots\cup D_h$, define the corresponding weights and biases:
    \begin{equation}
        V^{\ell}\equiv\begin{bmatrix}
        V^{\ell}_{D_1}&\dots&V^{\ell}_{D_h}
        \end{bmatrix}^T
        \qquad\text{and}\qquad
        u^{\ell}\equiv\begin{bmatrix}
        u^{\ell}_{D_1}&\dots&u^{\ell}_{D_h}
        \end{bmatrix}^T
    \end{equation}
    Furthermore, define the reconstruction matrix $P^{\ell}\in\real^{n^{\ell}\times h}$ with $P_{ij}=1$ if $i\in D_j$ and zero otherwise. Then, we can rewrite the upper bounds in Equation \ref{eq:icf_x_upper} for the whole layer as follows:
    \begin{equation}
    \label{eq:icf_x_upper_all}
        x^{\ell}
        \leq P^{\ell}\sigma(V^{\ell}x^{\ell-1}+u^{\ell}) + t^{\ell}
    \end{equation}
\end{definition}

Then, let us introduce interval inputs $x^{\ell-1}\in\interval^{n^{\ell-1}}$ and rewrite both upper and lower bounds:

\begin{definition}[Layer-Wise Abstraction]
\label{def:layer_abstraction}
    Assume that we are given input vectors $\underline{x}^{\ell-1}$ and $\overline{x}^{\ell-1}$, such that $x_i^{\ell-1}\in[\underline{x}_i^{\ell-1},\overline{x}_i^{\ell-1}]$. Define the following quantities:
    \begin{equation}
    \label{eq:interval_prop}
        \begin{bmatrix}
        \overline{r}^{\ell} \\
        \overline{t}^{\ell} \\
        \underline{t}^{\ell}
        \end{bmatrix}
        \equiv
        \begin{bmatrix}
        (V^{\ell})^+ & (V^{\ell})^- \\
        (A^{\ell}W^{\ell})^+ & (A^{\ell}W^{\ell})^- \\ (A^{\ell}W^{\ell})^- & (A^{\ell}W^{\ell})^+
        \end{bmatrix}
        \begin{bmatrix}
        \overline{x}^{\ell-1} \\
        \underline{x}^{\ell-1} \end{bmatrix}
        +
        \begin{bmatrix}
        u^{\ell} \\
        A^{\ell}b^{\ell} \\
        A^{\ell}b^{\ell}
        \end{bmatrix}
    \end{equation}
    where $(M)^+=\max(M,0)$ and $(M)^-=\min(M,0)$ element-wise. By construction, we can apply the interval propagation results from Equation \ref{eq:interval_prop} to further bound the values of Equations \ref{eq:icf_x_lower} and \ref{eq:icf_x_upper_all} as follows:
    \begin{equation}
    \label{eq:layer_output_bounds}
        \underline{x}^{\ell}\equiv\underline{t}^{\ell}
        \qquad\text{and}\qquad
        \overline{x}^{\ell}
        \equiv P^{\ell}\sigma(\overline{r}^{\ell}) + \overline{t}^{\ell}
    \end{equation}
    where $x_i^{\ell}\in[\underline{x}^{\ell}_i,\overline{x}^{\ell}_i]$ are the output intervals.
\end{definition}

In Figure \ref{fig:layer_abstract} we show a visual representation of the layer abstraction from Definition \ref{def:layer_abstraction}. Now, let us prove its main properties.

\begin{figure}[t]
\centering
    \includegraphics[width=\textwidth]{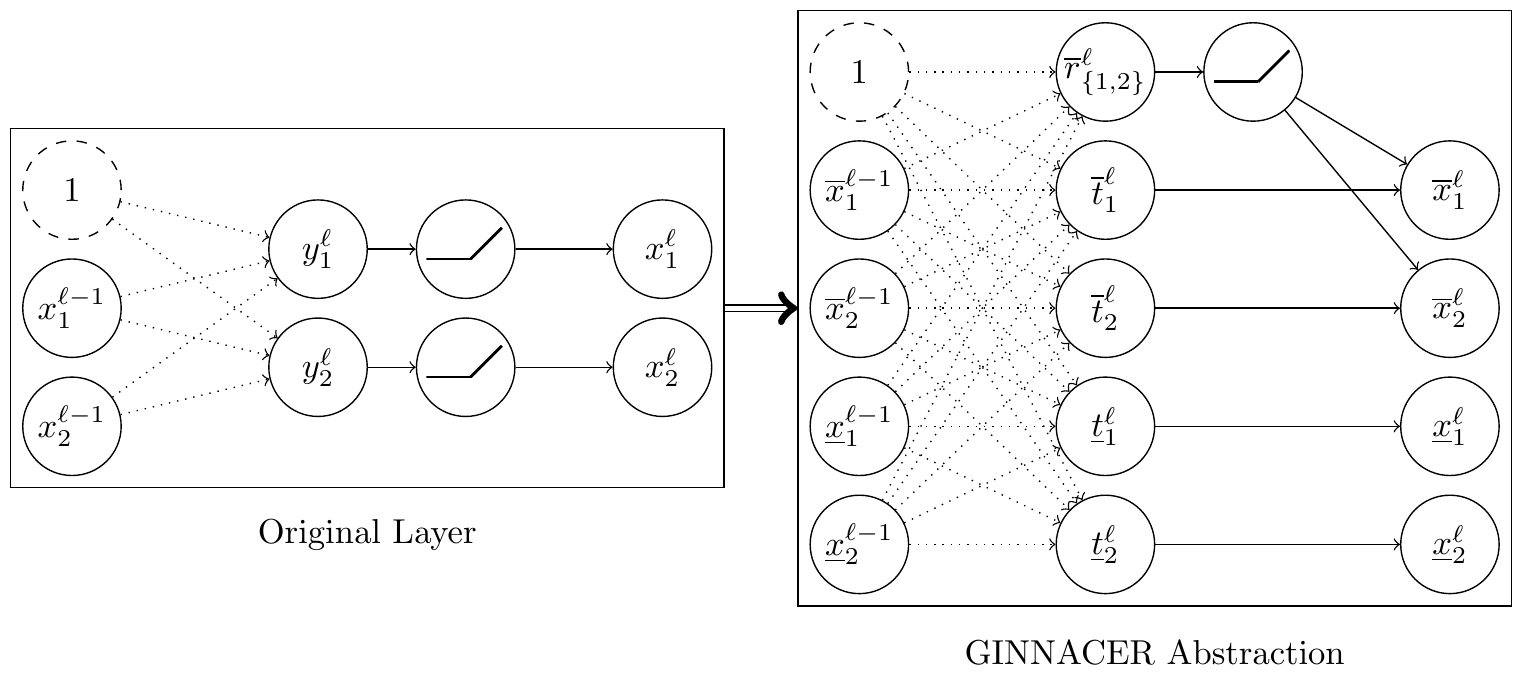}
\caption{Example of a GINNACER abstraction of a layer with two inputs $x^{\ell-1}$ and two outputs $x^{\ell}$. The dotted edges represent weighted connections, the dashed node represents the biases $b^{\ell}$. Here, we assume that the upper bound $\overline{r}^{\ell}_D$ in Equation \ref{eq:interval_prop} is shared by both neurons $D\equiv\{1,2\}$. As a result, the abstracted layer contains half the ReLU activations of the original one.}
\label{fig:layer_abstract}
\end{figure}

\begin{theorem}
\label{th:center_matching}
    The layer-wise abstraction in Definition \ref{def:layer_abstraction} yields matching output bounds at the centroid, i.e. when $\underline{x}^{\ell-1}=\overline{x}^{\ell-1}=x_c^{\ell-1}$.
\end{theorem}

\begin{proof}
    For any matching inputs $x^{\ell-1}=\underline{x}^{\ell-1}=\overline{x}^{\ell-1}$, Equation \ref{eq:interval_prop} simplifies to $\overline{r}^{\ell}=V^{\ell}x^{\ell-1}+u^{\ell}$ and $t^{\ell}=\overline{t}^{\ell}=\underline{t}^{\ell}=A^{\ell}(W^{\ell}x^{\ell-1}+b^{\ell})$. Because of Definitions \ref{def:upper_bound}, \ref{def:valid_subset}, \ref{def:valid_partition} and \ref{def:layer_upper_bound}, we know that $V^{\ell}x^{\ell-1}+u^{\ell}\leq0$ when $x^{\ell-1}=x^{\ell-1}_c$. Therefore, Equation \ref{eq:layer_output_bounds} yields $\underline{x}^{\ell}=\overline{x}^{\ell}=t^{\ell}$. 
\end{proof}

\begin{corollary}
\label{th:center_exact}
    When $\underline{x}^{\ell-1}=\overline{x}^{\ell-1}=x_c^{\ell-1}$, the matching bounds equal the original output $\underline{x}^{\ell}=\overline{x}^{\ell}=\sigma(W^{\ell}x^{\ell-1}_c+b^{\ell})$.
\end{corollary}

\begin{proof}
    In Theorem \ref{th:center_matching}, we established that $\underline{x}^{\ell}=\overline{x}^{\ell}=t^{\ell}$. Because of Definitions \ref{def:centact_matrix}, \ref{def:icf} and Theorem \ref{th:icf_equivalence}, we have $t^{\ell}=\sigma(W^{\ell}x^{\ell-1}_c+b^{\ell})$.
\end{proof}

\begin{theorem}
\label{th:soundness}
    The layer-wise abstraction in Definition \ref{def:layer_abstraction} is sound, i.e. for all $x^{\ell-1}\geq0$ such that $x^{\ell-1}_i\in[\underline{x}^{\ell-1}_i,\overline{x}^{\ell-1}_i]$ we have $x^{\ell}=\sigma(W^{\ell}x^{\ell-1}+b^{\ell})$ with $x^{\ell}_i\in[\underline{x}^{\ell}_i,\overline{x}^{\ell}_i]$.
\end{theorem}

\begin{proof}
    Because of Definitions \ref{def:layer_upper_bound}, we know that Equation \ref{eq:icf_x_upper_all} is a valid upper bound on $x^{\ell}$ for any specific $x^{\ell-1}\geq0$. When $x^{\ell-1}_i\in[\underline{x}^{\ell-1}_i,\overline{x}^{\ell-1}_i]$, we can write:
    \begin{align}\begin{split}
        x^{\ell}
        &\leq\max_{x^{\ell-1}}\Big\{P^{\ell}\sigma(V^{\ell}x^{\ell-1}+u^{\ell}) + A^{\ell}(W^{\ell}x^{\ell-1}+b^{\ell})\Big\}\\
        &\leq P^{\ell}\sigma\big(\max_{x^{\ell-1}}\{V^{\ell}x^{\ell-1}\}+u^{\ell}\big) + \max_{x^{\ell-1}}\{A^{\ell}W^{\ell}x^{\ell-1}\}+A^{\ell}b^{\ell}
    \end{split}\end{align}
    where each separate maximization problem reduces to taking $\underline{x}^{\ell-1}_i$ when the corresponding weights $V^{\ell}_{ij}$ or $(A^{\ell}W^{\ell})_{ij}$ are negative, and $\overline{x}^{\ell-1}_i$ otherwise. In matrix form, this can be written as Equation \ref{eq:interval_prop} shows. The lower bound in Equation \ref{eq:layer_output_bounds} is derived from Definition \ref{def:lower_bound} in similar fashion.
\end{proof}

\subsection{Composing the Abstract Layers}
\label{sec:full_abstraction}

At the beginning of Section \ref{sec:algorithm}, we mentioned that our algorithm computes a separate abstraction $g^{\ell}$ for each layer $\ell$. We define such layer-wise abstraction in Definition \ref{def:layer_abstraction}. Here, we prove that composing all the abstracted layers $g^{\ell}$ together satisfies the GINNACER requirements from Definition \ref{def:ginnacer}. First, let us take the input centroid from the GINNACER definition and propagate it through the whole neural network:

\begin{definition}[Centroid Propagation]
\label{def:centroid_prop}
    Given any centroid input $x^{-1}_c$ for the first pre-layer $\ell=0$, define $x^{\ell}_c\equiv\sigma(W^{\ell}x^{\ell-1}_c+b^{\ell})$ as the centroid of any following layer $\ell>0$.
\end{definition}

The sequence of centroids in Definition \ref{def:centroid_prop} allows us to align each abstracted layer. As a result, each of them introduces zero over-approximation margin for the given network input:

\begin{definition}[Full Abstraction]
\label{def:full_abstraction}
    Given the centroids $x^{\ell}_c$ from Definition \ref{def:centroid_prop} and any valid partition of $E^{\ell}$ for each layer $\ell$ (see Definition \ref{def:valid_partition}), define $g^{\ell}:\interval^{n^{\ell-1}}\to\interval^{n^{\ell}}$ as the corresponding layer-wise abstraction functions from Definition \ref{def:layer_abstraction}. Then, we can write a full abstraction as follows:
    \begin{equation}
        g^{\mathrm{full}}
        \equiv\sigma(W^0_{\mathrm{pre}}x^{-1})\circ g^1_{\mathrm{pre}}\circ g^2\circ\dots\circ g^m
    \end{equation}
    where we explicitly mark the two pre-layers from Section \ref{sec:neg_input_trick}, and we assume $g^1_{\mathrm{pre}}([x^0,x^0])$ accepts degenerate intervals as input.
\end{definition}

\begin{theorem}
    The abstraction in Definition \ref{def:full_abstraction} satisfies the GINNACER requirements from Definition \ref{def:ginnacer}, excluding the optional pre-layer $\ell=0$.
\end{theorem}

\begin{proof}
    First, $g^{\mathrm{full}}$ is globally sound as it is the composition of an equivalent transformation (the pre-layer $\ell=0$) and several sound abstractions $g^{\ell}$ (see Theorem \ref{th:soundness}). Second, $g^{\mathrm{full}}$ never contains more non-linearities than $f$ (excluding the pre-layer $\ell=0$) since the neurons $E^{\ell}$ of each layer $\ell>0$ are partitioned in only $h\leq n^{\ell}$ subsets (see Definition \ref{def:valid_partition}). Third, $g^{\mathrm{full}}$ yields center-exact reconstruction since each layer $\ell\in[0,m]$ yields center-exact reconstruction (see Corollary \ref{th:center_exact}).
\end{proof}

\subsection{Pseudocode}
\label{sec:algo_code}

Here, summarise our algorithmic contribution by introducing Algorithms \ref{alg:ginnacer_params} and \ref{alg:ginnacer_exec}, which contains the pseudocode of GINNACER. More specifically, Algorithm \ref{alg:ginnacer_params} computes all the parameters $G$ of the abstraction. The presence of potential negative inputs (see Section \ref{sec:neg_input_trick}) is handled in Lines $2-8$. The inactive canonical form (see Section \ref{sec:inactive_form}) is used in Line 10. The number of non-linearities (see Section \ref{sec:clustering}) is reduced in Line 11, by computing a valid neuron partition (see also \ref{sec:partioning}, Algorithm \ref{alg:neuron_partition}). The parameters of the corresponding upper bounds (see Section \ref{sec:interval_prop}) are computed in Line 12. The centroid for the next layer (see Section \ref{sec:full_abstraction}) is prepared in Line 13.

\begin{algorithm}[t]
\caption{Compute GINNACER Parameters}
\label{alg:ginnacer_params}
\begin{algorithmic}[1]
\Require weights $(W^1,\dots,W^m)$, biases $(b^1,\dots,b^m)$, centroid $x^0_c$, negative input flag $neg\_in$ (see Definition \ref{def:pre_layers}).
\Ensure GINNACER parameters $G$.
\State $G\gets()$
\If{$neg\_in$}
    \State $W^0\gets[I,-I]^T$
    \State $b^0\gets[0,\dots,0]^T$
    \State $G\gets(\{W^0,b^0\})$\hfill\Comment{store pre-layer $\ell=0$ parameters}
    \State $W^1\gets[W^1,-W^1]$\hfill\Comment{modify (pre-)layer $\ell=1$ parameters}
    \State $x^0_c\gets \sigma(W^0x^0_c+b^0)$\hfill\Comment{pass the centroid through pre-layer $\ell=0$}
\EndIf
\ForAll{$\ell\in[1,m]$}
    \State $\{A^{\ell},S^{\ell}\}\gets\text{CentroidAct}(W^{\ell},b^{\ell},x^{\ell-1}_c)$\hfill\Comment{Definition \ref{def:centact_matrix}}
    \State $\{D_1,\dots,D_h\}\gets\text{ValidPartition}(S^{\ell}W^{\ell},S^{\ell}b^{\ell},x^{\ell-1}_c)$\hfill\Comment{Definition \ref{def:valid_partition}}
    \State $\{P^{\ell},V^{\ell},u^{\ell}\}\gets\text{UpBound}(S^{\ell}W^{\ell},S^{\ell}b^{\ell},\{D_1,\dots,D_h\})$\hfill\Comment{Definition \ref{def:layer_upper_bound}}
    \State $x^{\ell}_c\gets\sigma(W^{\ell}x^{\ell-1}_c+b^{\ell})$\hfill\Comment{Definition \ref{def:centroid_prop}}
    \State $G\gets(G,\{P^{\ell},V^{\ell},u^{\ell},A^{\ell}W^{\ell},A^{\ell}b^{\ell}\})$\hfill\Comment{store layer $\ell$ parameters}
\EndFor
\end{algorithmic}
\end{algorithm}

Given the parameters $G$, Algorithm \ref{alg:ginnacer_exec} computes the abstracted output for a specific input $x^0$. Such input is converted into degenerate intervals in Lines 1-5. Then, the layer-wise abstraction from Definition \ref{def:layer_abstraction} is executed in Lines 6-12. After propagating the intervals over all the layer, the output abstraction $\left[\underline{x}^m,\overline{x}^m \right]$ is given as a result.

\begin{algorithm}[t]
\caption{Execute GINNACER Abstraction}
\label{alg:ginnacer_exec}
\begin{algorithmic}[1]
\Require GINNACER parameters $G$ (see Algorithm \ref{alg:ginnacer_params}), concrete input $x^0$, negative input flag $neg\_in$ (see Definition \ref{def:pre_layers}).
\Ensure interval bounds $\underline{x}^m$ and $\overline{x}^m$.
\If{$neg\_in$}
    \State $\underline{x}^0,\overline{x}^0\gets\sigma(W^0x^0+b^0)$\hfill\Comment{compute exact pre-layer $\ell=0$}
\Else
    \State $\underline{x}^0,\overline{x}^0\gets x^0$\hfill\Comment{degenerate input intervals (Definition \ref{def:full_abstraction})}
\EndIf
\ForAll{$\ell\in[1,m]$}
    \State $\overline{r}^{\ell}\gets(V^{\ell})^+\overline{x}^{\ell-1}+(V^{\ell})^-\underline{x}^{\ell-1}+u^{\ell}$\hfill\Comment{split positive}
    \State $\overline{t}^{\ell}\gets(A^{\ell}W^{\ell})^+\overline{x}^{\ell-1}+(A^{\ell}W^{\ell})^-\underline{x}^{\ell-1}+A^{\ell}b^{\ell}$\hfill\Comment{and negative}
    \State $\underline{t}^{\ell}\gets(A^{\ell}W^{\ell})^+\underline{x}^{\ell-1}+(A^{\ell}W^{\ell})^-\overline{x}^{\ell-1}+A^{\ell}b^{\ell}$\hfill\Comment{matrix entries}
    \State $\overline{x}^{\ell}\gets P^{\ell}\sigma(\overline{r}^{\ell})+\overline{t}^{\ell}$
    \State $\underline{x}^{\ell}\gets\underline{t}^{\ell}$
\EndFor
\end{algorithmic}
\end{algorithm}

\subsection{Time Complexity}
\label{sec:algo_complexity}

We conclude our theoretical discussion by analysing the time complexity of GINNACER for a single layer $\ell$. Recall that each layer $\ell$ has $n^{\ell-1}$ inputs and $n^{\ell}$ neurons. According to Algorithm \ref{alg:ginnacer_params}, we need to execute the following operations when computing a GINNACER abstraction:

\begin{itemize}
    \item Computing the activations at the centroid (Line 10) require establishing whether the condition $(W^{\ell}x_c^{\ell-1}+b^{\ell})_i\geq0$ holds for each neuron $i\in[1,n^{\ell}]$. This takes $O(n^{\ell}n^{\ell-1})$ multiplications, $O(n^{\ell}\log_2n^{\ell-1})$ additions and $O(n^{\ell})$ comparison operations.

    \item Computing a valid partition (Line 11) is implementation-dependent. Our approach (see \ref{sec:partioning}, Algorithm \ref{alg:neuron_partition}) computes $n^{\ell}(n^{\ell}-1)/2$ candidate merging operations in the worst case (no neuron can be merged) and $n^{\ell}-1$ in the best case (all neurons are merged). If we keep track of each subset weights $V_D^{\ell}$ and biases $u_D^{\ell}$ during the execution of the algorithm, each merging operation requires $O(n^{\ell-1})$ max operations. Furthermore, computing whether the condition $(V_D^{\ell})^Tx_c^{\ell-1}+u_D^{\ell}\leq0$ holds requires $O(n^{\ell-1})$ multiplications, $O(log_2n^{\ell-1})$ additions and one comparison operations. Similarly, constructing the permutation matrix $P^{\ell}$ (Line 13) requires $O(n^{\ell})$ assignments.
    
    \item Propagating the centroid (Line 13) requires computing $x_c^{\ell}=\sigma(W^{\ell}x_c^{\ell-1}+b^{\ell})$. This takes $O(n^{\ell}n^{\ell-1})$ multiplications, $O(n^{\ell}\log_2n^{\ell-1})$ additions and $O(n^{\ell})$ max operations for the activation function.
\end{itemize}

In summary, the abstraction time of GINNACER is dominated by the necessity of computing $O((n^{\ell})^2)$ candidate merging operations. Each one of them takes $O(n^{\ell-1})$ operations, yielding a total per-layer complexity of $O((n^{\ell})^2n^{\ell-1})$.

In contrast, the theoretical complexity of executing GINNACER at inference time (see Algorithm \ref{alg:ginnacer_exec}) is not greater than that of the original network. In fact, each affine operation in Lines 7-9 requires two (sparse) matrix-vector products, with $O(n^{\ell}n^{\ell-1})$ multiplications and $O(n^{\ell}\log_2n^{\ell-1})$ additions. Furthermore, Line 10 requires $O(n^{\ell})$ max operations for the ReLUs, $O(n^{\ell})$ copy operations ($P^{\ell}$ is a permutation matrix) and $O(n^{\ell})$ additions. Overall, the complexity is dominated by the $O(n^{\ell}n^{\ell-1})$ matrix-vector products, as it is for the original layer $\sigma(W^{\ell}x^{\ell-1}+b^{\ell})$. In practice, our experiments of Section \ref{sec:time_comparison} show that some execution overhead is to be expected.

\section{Empirical Analysis}
\label{sec:experiments}

In this section, we complement our previous qualitative comparison (see Section~\ref{sec:motivating_example}) with a rigorous empirical analysis of the behaviour of neural network abstraction methods. Our objective is twofold:
\begin{itemize}
    \item Compare our GINNACER algorithm described in Section \ref{sec:algorithm} to the existing state-of-the-art methods. We do so in Section \ref{sec:empirical_comparison}.
    \item Quantify the ability of our algorithm from Section \ref{sec:algorithm} to reduce the number of ReLU activation functions. We do so in Section \ref{sec:empirical_clustering}.
\end{itemize}
Before addressing these questions, let us introduce our experimental setting.

\subsection{List of Neural Networks}
\label{sec:list_of_nn}

In our experiments, we analyze the performance of abstraction methods over the following set of neural networks. These networks cover a wide range of applications: aircraft control, image classification and anomaly detection.
\begin{itemize}
    \item \textbf{Acas-Xu.} A collection of 45 ReLU networks for unmanned aircraft collision avoidance with six hidden layers of 50 neurons each.\footnote{\label{foot:vvn_comp}\url{https://github.com/stanleybak/vnncomp2021/tree/main/benchmarks}} The five inputs encode the relative position and trajectory of an incoming aircraft, whereas the outputs predict which of the five possible piloting decisions needs to be taken to avoid a collision \cite{Julian2016}.
    
    \item \textbf{MNIST (fc).} An image classification network from the VNN-COMP'21 suite with six layers of 256 neurons each.\footref{foot:vvn_comp} The inputs are gray-scale images of handwritten digits, with 28x28 pixels rescaled to the interval [0,1]. The ten outputs represent prediction scores on which digit $\{0\dots9\}$ is most likely to be in the input \cite{Lecun1998}.
    
    \item \textbf{ToyADMOS.} A ReLU autoencoder with 4 layers of 128 neurons, a bottleneck layer with just 8 neurons, and another 4 layers of 128 neurons for a total of nine layers (batch normalisation layers are absorbed into the fully-connected ones).\footnote{\label{foot:tiny_ml}\url{https://github.com/mlcommons/tiny/tree/master/benchmark/training}} The input is a 640 feature vector representing the spectrogram of a short audio excerpt of industrial machinery sound. This network is routinely used to test neural network compression algorithms as part of the MLPerf Tiny Benchmark \cite{Banbury2021}.

    \item \textbf{Visual Wake Words.} A medium-sized convolutional network for recognising the presence of people in 96x96 pixel images with three colour channels.\footref{foot:tiny_ml} For our experiments, we unroll the last 16 layers to their equivalent fully-connected form, and store the weights in sparse matrix format. The number of neurons per layer is reported in Table \ref{tab:network_sizes}. This network is also part of the MLPerf Tiny Benchmark suite \cite{Banbury2021}.
\end{itemize}

\begin{table}[t]
\centering
    \resizebox{\linewidth}{!}{
    \setlength\tabcolsep{2pt} 
    \begin{tabular}{ |c|cccccccccccccccc| } 
        \hline
        & \multicolumn{16}{|c|}{ReLUs per Layer} \\
        \cline{2-17}
        & L1 & L2 & L3 & L4 & L5 & L6 & L7 & L8 & L9 & L10 & L11 & L12 & L13 & L14 & L15 & L16 \\
        \hline
        AcasXu (all) & 50 & 50 & 50 & 50 & 50 & 50 & & & & & & & & & & \\
        MNIST (fc) & 256 & 256 & 256 & 256 & 256 & 256 & & & & & & & & & & \\
        ToyADMOS & 128 & 128 & 128 & 128 & 8 & 128 & 128 & 128 & 128 & & & & & & & \\
        VisualWW & 2304 & 4608 & 4608 & 4608 & 4608 & 4608 & 4608 & 4608 & 4608 & 4608 & 4608 & 4608 & 1152 & 2304 & 2304 & 2304 \\
        \hline
    \end{tabular}}
\caption{Size of the networks in our experiment suite. The output layers are omitted as they do not contain ReLU non-linearities.}
\label{tab:network_sizes}
\end{table}

\subsection{List of State-of-the-art Algorithms}
\label{sec:list_of_sota}

In our experiments, we compare against the following state-of-the-art algorithms:
\begin{itemize}
    \item \textbf{Interval Networks \citep{Prabhakar2019}.} This algorithm converts the original neural network to an equivalent network capable of handling interval arithmetic. Then, the random pairs of neurons in the same layer are merged: a ``right'' abstraction and ``left'' abstraction are needed to correctly over-approximate the weights and biases of the layer under consideration and the next one. This algorithm cannot handle negative inputs; thus, we apply our pre-processing step from Section \ref{sec:neg_input_trick} for a fair comparison.
    \item \textbf{Network Abstractions \citep{Elboher2020}.} This algorithm creates four copies of the original neurons, depending on whether they are upper (``Inc'') or lower (``Dec'') bounds and whether their output weights are all positive (``Pos'') or negative (``Neg''). Neurons belonging to the same layer and type (e.g. ``Pos-Inc'') can then be merged by taking the max or min of their weights and biases. This algorithm cannot handle negative inputs either; thus, we apply our pre-processing step from Section \ref{sec:neg_input_trick} in our experiments.
    \item \textbf{FastLin Bounds \citep{Weng2018}.} This algorithm provides a pair of linear upper and lower bounds on each ReLU activation function. Crucially, the algorithm requires knowing the maximum and minimum possible inputs of each ReLU, which are computed by offline propagation of intervals through the neural network. The FastLin bounds need to be re-computed whenever an input exceeds the expected range. This is a major weakness of all local abstraction techniques used in state-of-the-art neural network verifiers~\citep{Katz2019,Salman2019,Bak2021}.
\end{itemize}
In all our experiments, we set the parameters of the above algorithms as follows. For the two global abstraction algorithms in \cite{Prabhakar2019,Elboher2020}, we set the number of neurons in each layer to be equal to the number of ReLU non-linearities GINNACER uses.\footnote{When this is not possible, the minimum larger number is chosen.} For the local abstraction in \cite{Weng2018}, we center the input domain around the same centroid $c$ as GINNACER, and bound each input in the hypercube $x_i\in[c_i-\epsilon,c_i+\epsilon]$.

\subsection{Over-Approximation Comparison}
\label{sec:empirical_comparison}

In this section, we compare the ability of our GINNACER algorithm to abstract the networks listed in Section~\ref{sec:list_of_nn} against the state-of-the-art. More specifically, we choose three different centroids: the centre of properties 1 and 2 in the VNN-COMP'21 competition for the AcasXu 1-1 network, a random image from the MNIST dataset, a random spectrogram from the ToyADMOS test set, and a random image from the COCO 2014 training set for the Visual Wake Word network.\footnote{\url{https://cocodataset.org/\#download}} Then, we sample 10000 input points on the surface of a hypercube of size $2\delta$ centred around centroid $c$. For each abstraction algorithm, we report the largest over-approximation margin (worst-case scenario) across the 10000 samples. Note that for the FastLin abstraction \citep{Weng2018} we set parameter $\epsilon=\delta$, thus testing the smallest possible abstraction for each value of $\delta$.

\begin{figure}[t]
\centering
    \begin{subfigure}[b]{0.49\textwidth}
    \centering
        \includegraphics[width=\textwidth]{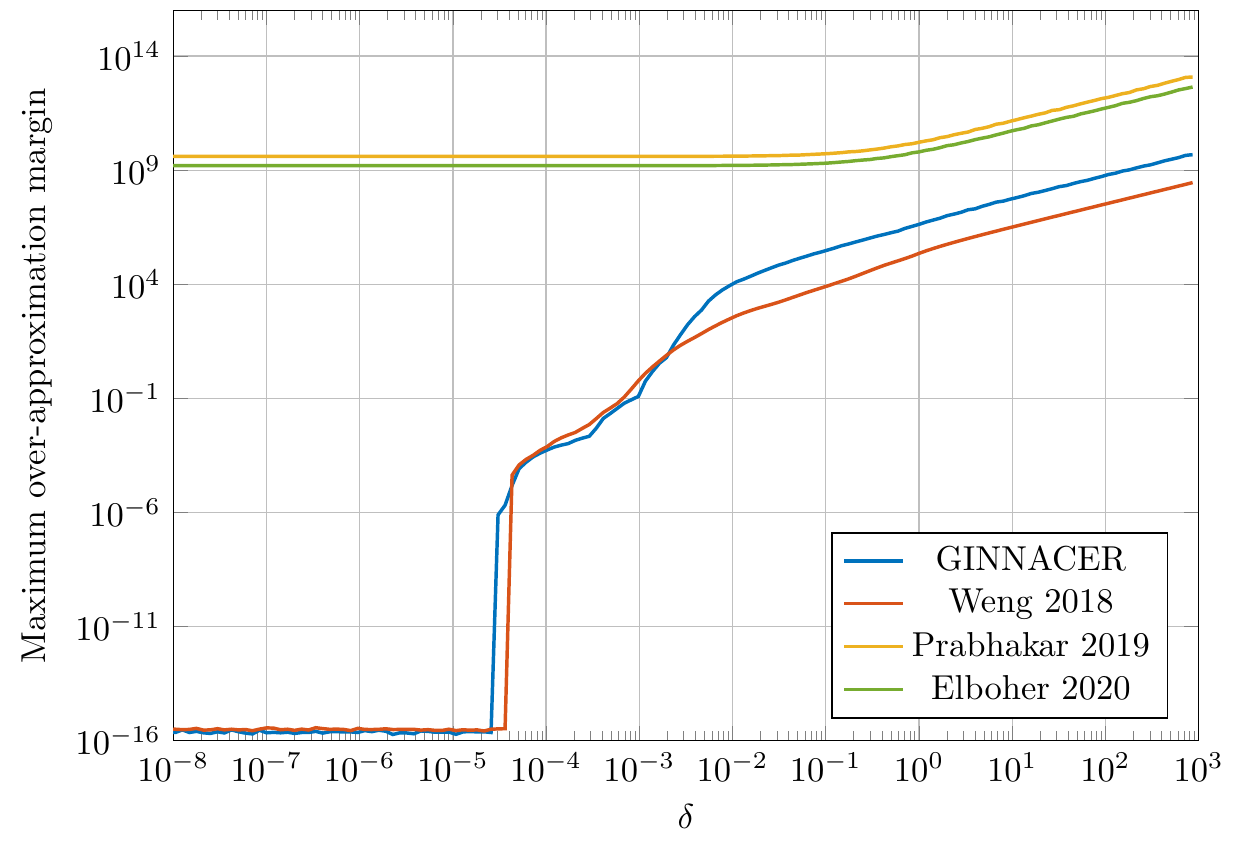}
        \caption{AcasXu (1-1)}
        \label{fig:acas_xu_dist}
    \end{subfigure}
    \begin{subfigure}[b]{0.49\textwidth}
    \centering
        \includegraphics[width=\textwidth]{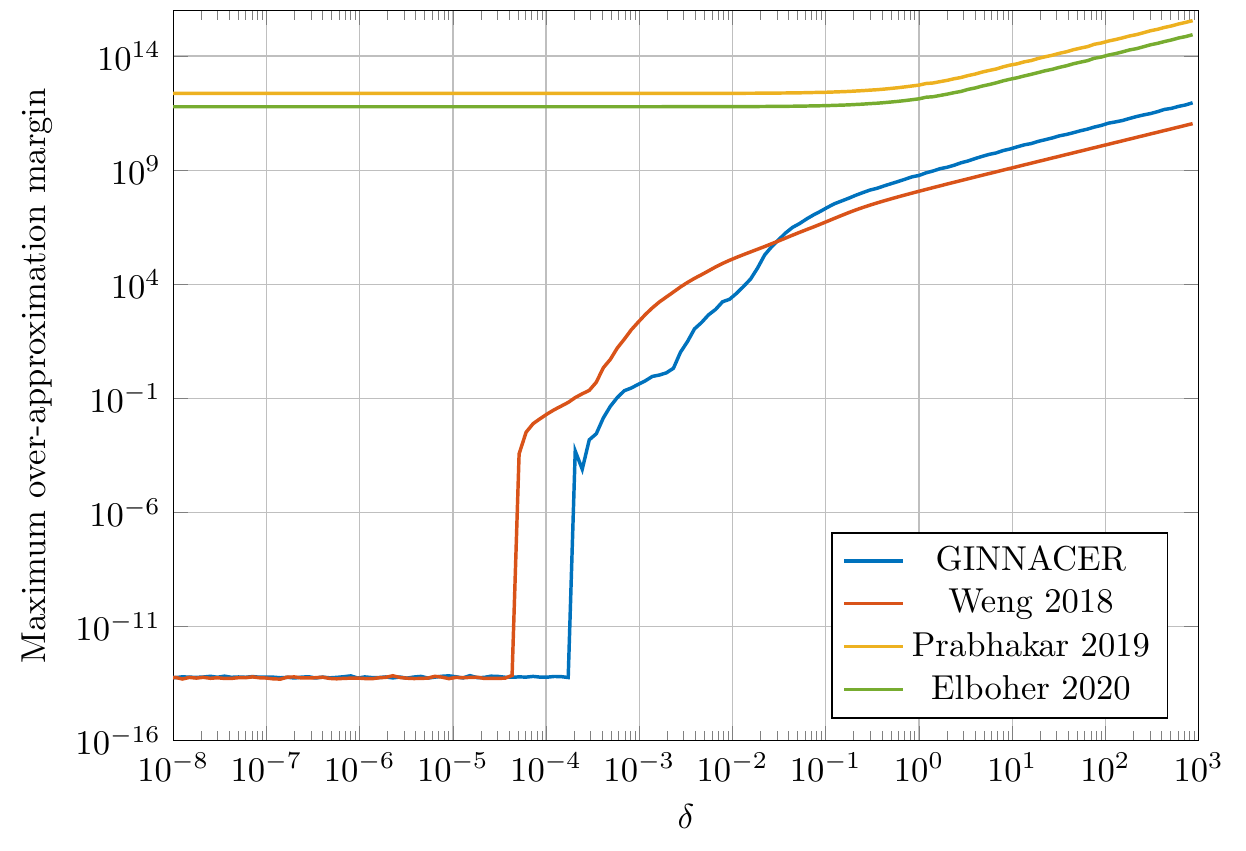}
        \caption{MNIST (fc)}
        \label{fig:mnist_fc_dist}
    \end{subfigure}
    \begin{subfigure}[b]{0.49\textwidth}
    \centering
        \includegraphics[width=\textwidth]{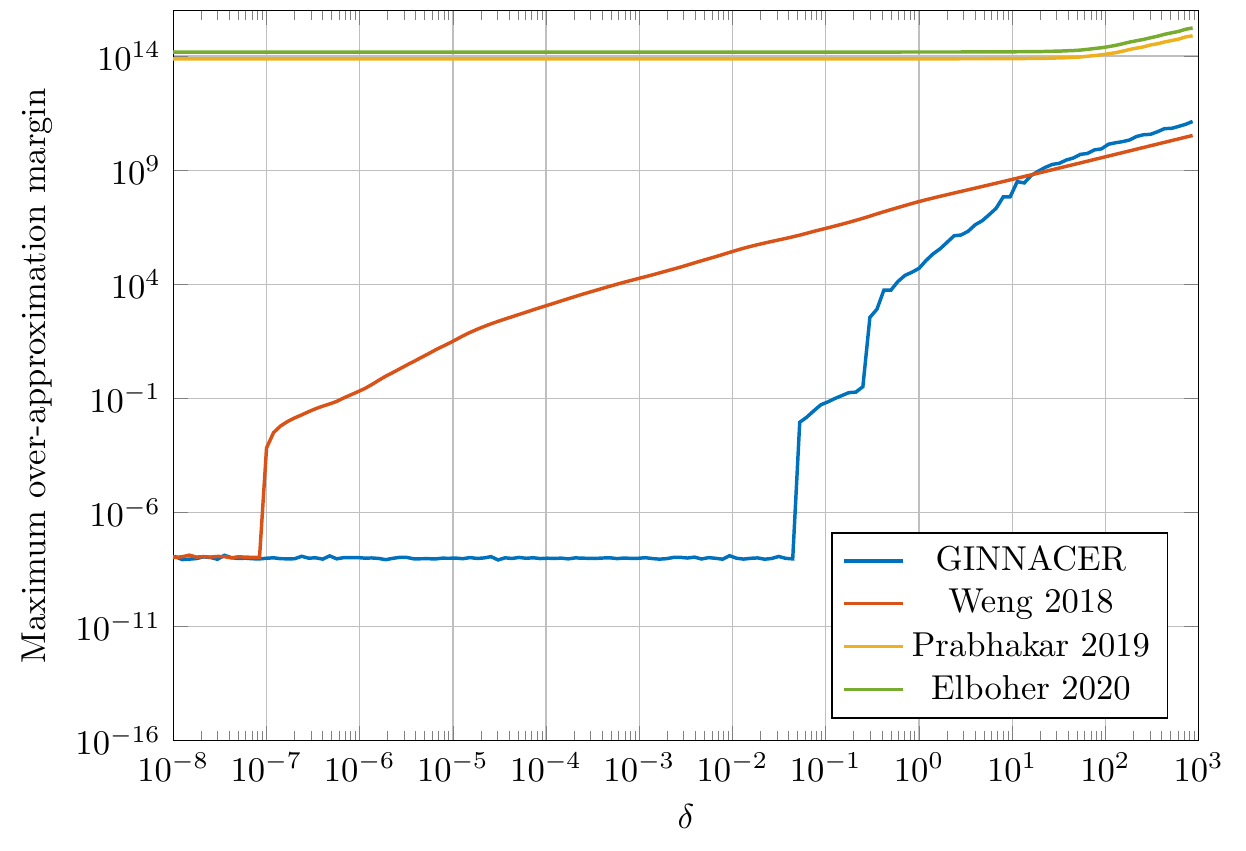}
        \caption{ToyADMOS}
        \label{fig:toy_admos_dist}
    \end{subfigure}
    \begin{subfigure}[b]{0.49\textwidth}
    \centering
        \includegraphics[width=\textwidth]{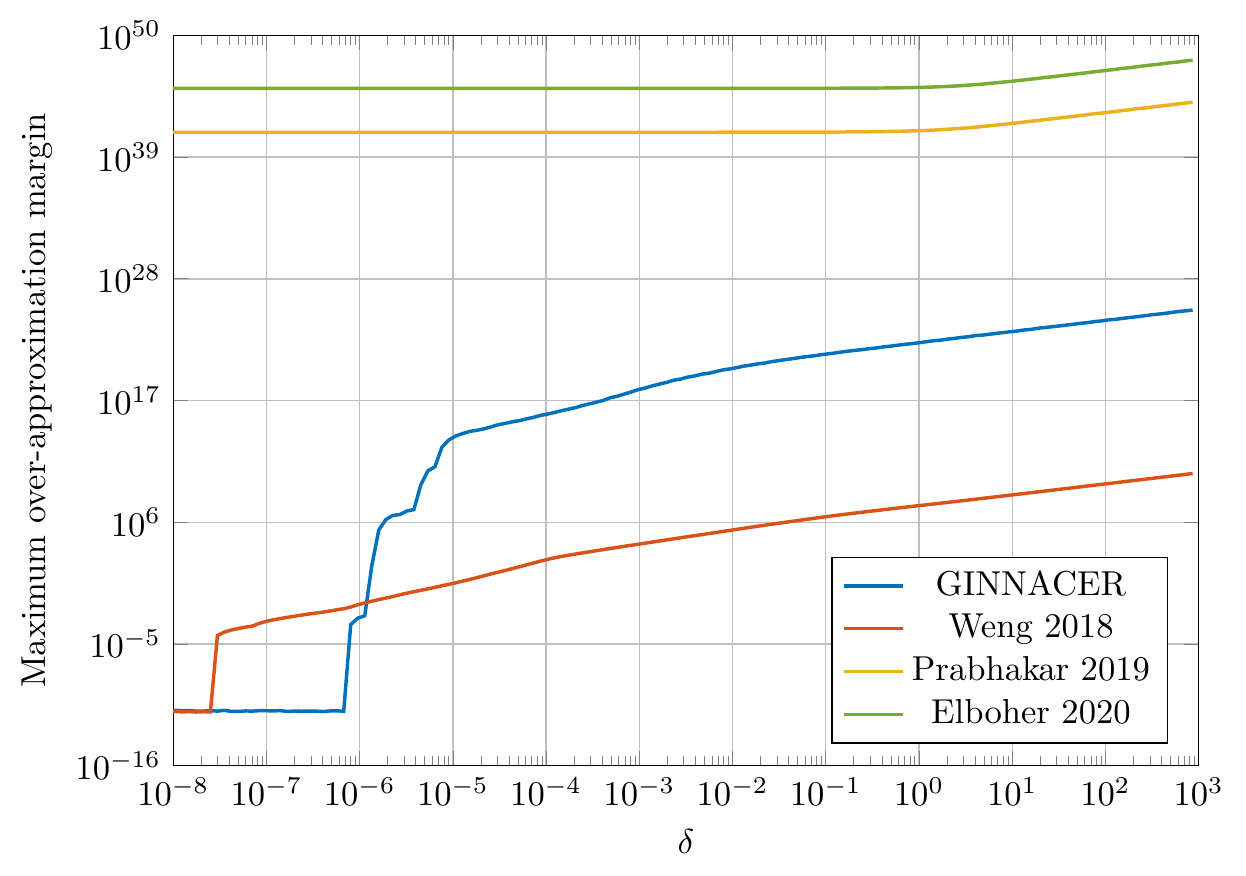}
        \caption{Visual Wake Words}
        \label{fig:vww_96_dist}
    \end{subfigure}
\caption{Comparison between GINNACER and state-of-the-art abstraction algorithms. Note that we use the tightest FastLin abstraction \citep{Weng2018} for each $\delta$.}
\label{fig:distance_comp}
\end{figure}

The results in Figure \ref{fig:distance_comp} complement the qualitative example in Figure \ref{fig:abs_comp}. Namely, we see that the existing global abstraction algorithms \citep{Prabhakar2019,Elboher2020} produce extremely large over-approximation margins. In contrast, GINNACER is on par with the local FastLin algorithm \cite{Weng2018}, while producing better over-approximation margins for smaller deviations from the centroid $c$. On the larger layers of the Visual Wake Words network and large deviations $\delta$, the FastLin algorithm achieves the smallest over-approximation margins. However, note that all three global abstraction algorithms use the same abstraction for all values of $\delta$, whereas FastLin is specifically recomputed for each $\delta$ (implicitly making it a quadratic abstraction).

Finally, the experiments of Figure \ref{fig:distance_comp} show the over-approximation margin of GINNACER when all layers are abstracted. For the sake of completeness, in \ref{sec:tradeoff} we show what happens when we only abstract a subset of the layers.

\subsection{Time Comparison}
\label{sec:time_comparison}

In this section, we measure the time required to compute each abstraction and use it at inference time. For the sake of clarity, we use the same experimental setting of our comparison in Section \ref{sec:empirical_comparison}. The only differences are the following: we recompute the abstraction 13 times to collect dispersion statistics, and independently measure the inference time 1000 times for each abstraction run. Care was taken to reduce the measurement noise to the minimum.\footnote{The experiments were run on a CentOS Linux 7 server with two Intel(R) Xeon(R) CPU E5-2620 v4 @ 2.10GHz, which carry 8 physical cores each. In order to avoid context switching, we used no more than 15 cores at the same time.}

\begin{figure}[p]
\centering
    \begin{subfigure}[b]{0.49\textwidth}
    \centering
        \includegraphics[width=\textwidth]{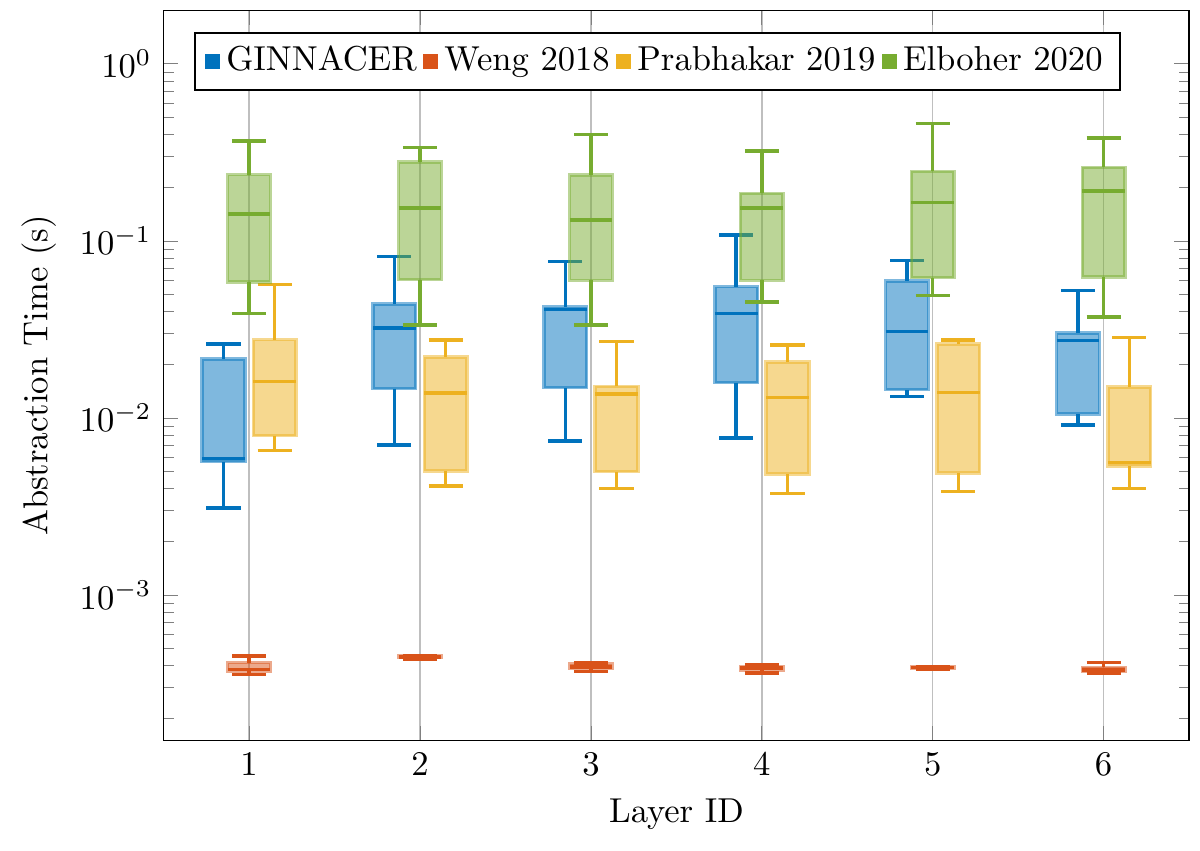}
        \caption{AcasXu (1-1)}
        \label{fig:time_acasxu}
    \end{subfigure}
    \begin{subfigure}[b]{0.49\textwidth}
    \centering
        \includegraphics[width=\textwidth]{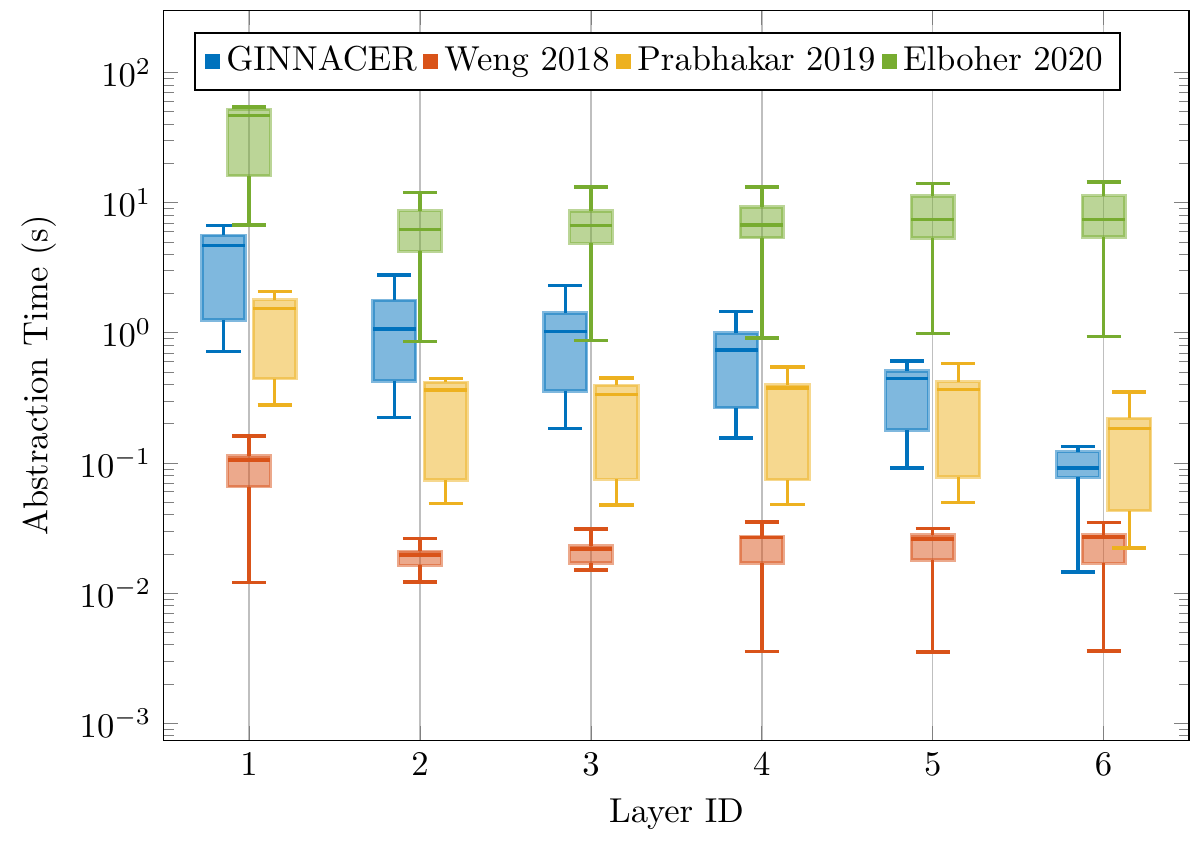}
        \caption{MNIST (fc)}
        \label{fig:time_mnist_fc}
    \end{subfigure}
    \begin{subfigure}[b]{0.62\textwidth}
    \centering
        \includegraphics[width=\textwidth]{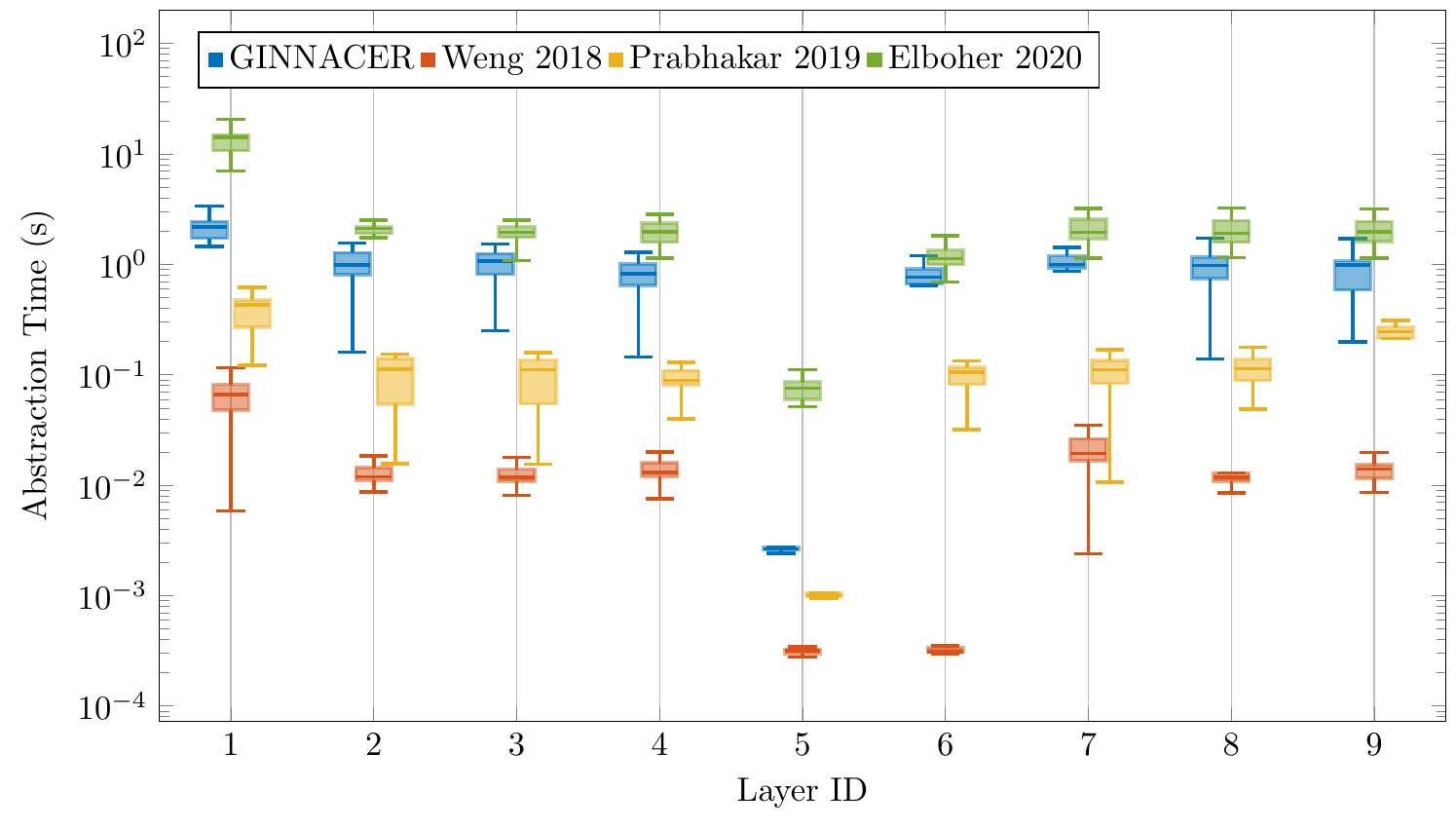}
        \caption{ToyADMOS}
        \label{fig:time_anomaly}
    \end{subfigure}
    \begin{subfigure}[b]{0.98\textwidth}
    \centering
        \includegraphics[width=\textwidth]{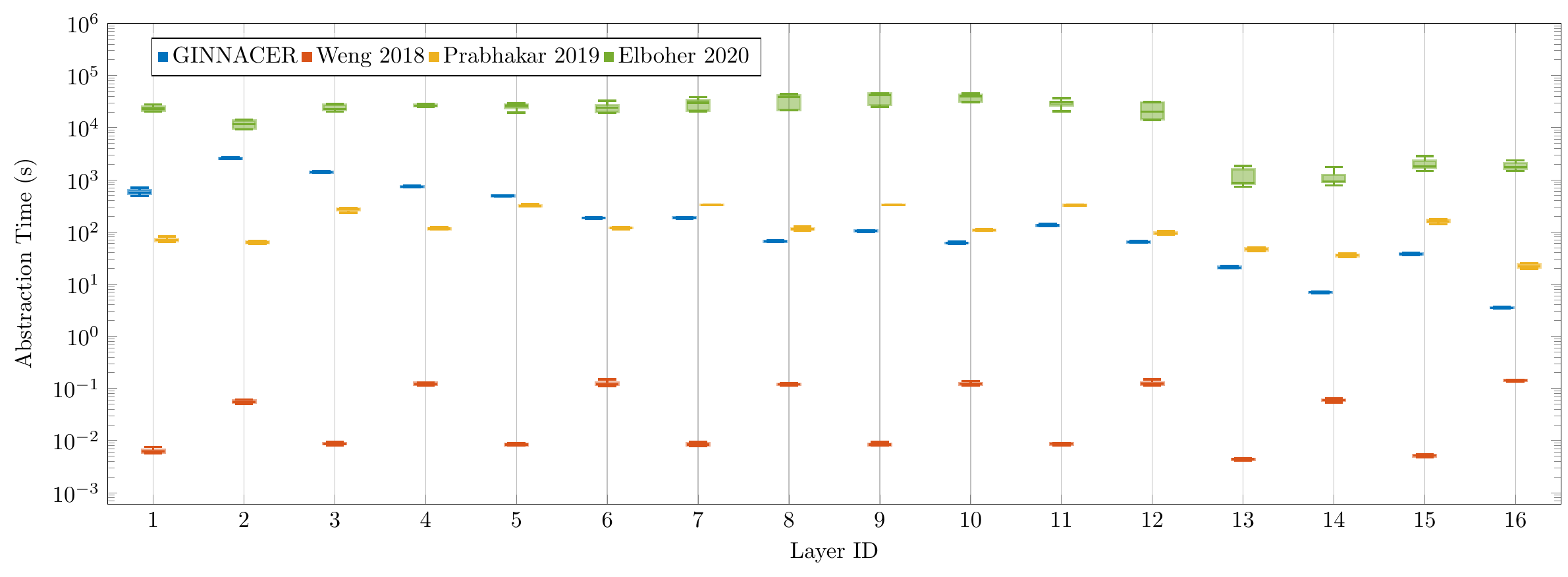}
        \caption{Visual Wake Words}
        \label{fig:time_vww}
    \end{subfigure}
\caption{Abstraction time over all network layers in our benchmark suite. The dispersion statistics (first, second and third quartile, $\pm1.5IQR$) are computed over 13 runs.}
\label{fig:compression_time}
\end{figure}

The abstraction time measurements are summarised in Figure \ref{fig:compression_time}. There, we report several dispersion statistics: namely, each box represent the first, second (median) and third quartile, whereas the whiskers represent the closest data points to $\pm1.5IQR$.\footnote{See \texttt{pgfplots} manual, Section 5.12, Page 506 at \url{https://ctan.org/pkg/pgfplots}} We separate the result per each layer of the neural networks, as their size varies across our experimental suite.

The results show a spread of the abstraction time across several orders of magnitude. Overall, the local FastLin algorithm is the fastest. This is because it performs operations for each individual neuron independently. In contrast, the three global abstraction algorithms need to reason about pairs of neurons, which slows down their abstraction time considerably. Among these three, GINNACER is roughly the second fastest one, after the Interval Networks of \citep{Prabhakar2019}. This empirical result reflects the theoretical time complexity analysis in \ref{sec:partioning}.

\begin{table}[t]
\centering
    \resizebox{\linewidth}{!}{
    \begin{tabular}{ |c|c|c|c|c| } 
        \hline
        & AcasXu (1-1) & MNIST (fc) & ToyADMOS & VisualWW \\
        \hline
        Original & \SI{9.86e-05}{\second}
                 & \SI{4.78e-03}{\second}
                 & \SI{5.50e-03}{\second}
                 & \SI{1.25e-02}{\second} \\
        GINNACER & \SI{7.75e-04}{\second}
                 & \SI{2.07e-02}{\second}
                 & \SI{2.69e-02}{\second}
                 & \SI{1.37e-02}{\second} \\
        Weng 2018 & \SI{3.68e-04}{\second}
                  & \SI{1.65e-02}{\second}
                  & \SI{1.46e-02}{\second}
                  & \SI{5.23e-02}{\second} \\
        Prabhakar 2019 & \SI{7.59e-04}{\second}
                       & \SI{6.66e-02}{\second}
                       & \SI{5.85e-02}{\second}
                       & \SI{2.46e-01}{\second} \\
        Elboher 2020 & \SI{1.09e-02}{\second}
                     & \SI{2.94e-02}{\second}
                     & \SI{3.10e-02}{\second}
                     & \SI{1.27e+00}{\second} \\
        \hline
    \end{tabular}}
\caption{Average execution time of the original and abstracted networks over 13K inference runs. All measurements below $1\times 10^{-1}\text{ s}$ have standard deviation in the same order of magnitude as their corresponding average.}
\label{tab:execution_time}
\end{table}

On a different note, we report the average inference times in Table \ref{tab:execution_time}. Note that these measurements have two limitations. First, the standard deviation is almost always in the same order of magnitude as the empirical average. This is because the inference time on these networks is practically instant, which adds a considerable amount of noise to our measurements. Second, our Python code was optimised to minimise abstraction time, not inference time. We believe faster inference times for all abstracted networks are possible.

At the same time, we can still notice a few general trends. On the one hand, abstracted neural network are slower than the original one since they require more computation to maintain valid bounds on the original variables. On the other hand, the abstraction in \cite{Elboher2020} appears to be the slowest, as it quadruples the number of neurons required per each layer.

\subsection{Effectiveness of our Partition Strategy}
\label{sec:empirical_clustering}

In this section, we quantify the ability of our GINNACER algorithm to reduce the number of ReLU non-linearities in the abstracted network. Recall that GINNACER has the requirement of exact reconstruction at the centroid (see Section \ref{sec:problem_statement}), which we satisfy by rejecting candidate subsets that yield an active ReLU at the centroid (see Section \ref{sec:clustering}). Consequently, GINNACER cannot cluster together all the neurons in a layer, like other global abstraction algorithms can do \citep{Prabhakar2019,Elboher2020}.

In our experiments, we run the GINNACER algorithm multiple times with different centroids: for the 45 AcasXu networks, we use the input centres of the ten safety properties used in the VNN-COMP'21 competition, whereas for MNIST and ToyADMOS we sample 1000 inputs from their respective datasets. Similarly, for Visual Wake Words we sample 1000 images from the COCO 2014 training set, equally split between containing a person or not.\footnote{\url{https://cocodataset.org/\#download}.}

\begin{figure}[t]
\centering
    \begin{subfigure}[b]{0.29\textwidth}
    \centering
        \includegraphics[width=\textwidth]{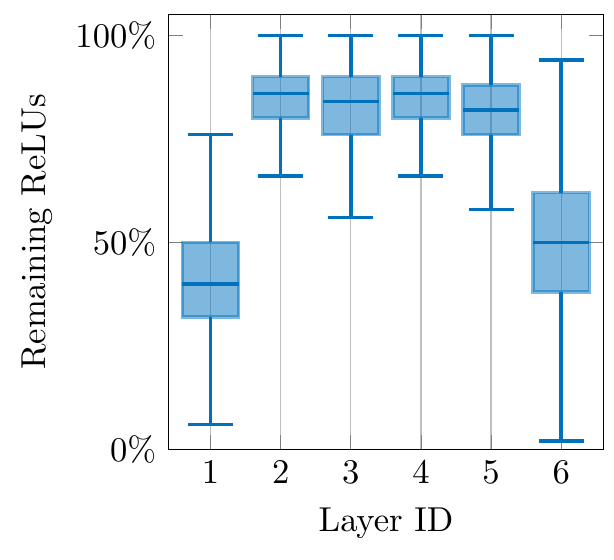}
        \caption{AcasXu (all)}
        \label{fig:bar_acasxu}
    \end{subfigure}
    \begin{subfigure}[b]{0.29\textwidth}
    \centering
        \includegraphics[width=\textwidth]{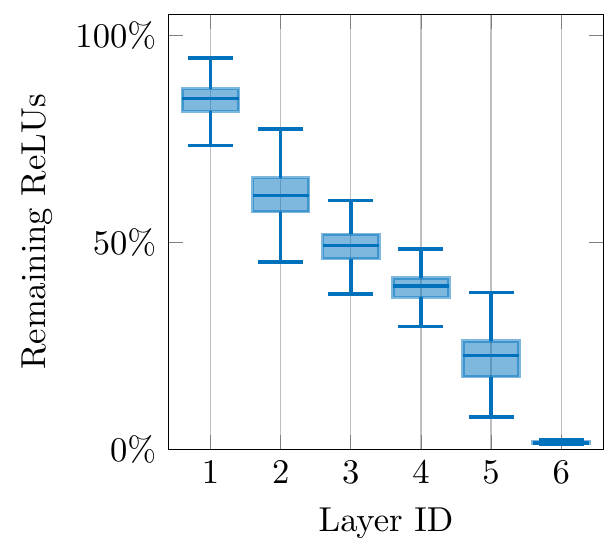}
        \caption{MNIST (fc)}
        \label{fig:bar_mnist_fc}
    \end{subfigure}
    \begin{subfigure}[b]{0.39\textwidth}
    \centering
        \includegraphics[width=\textwidth]{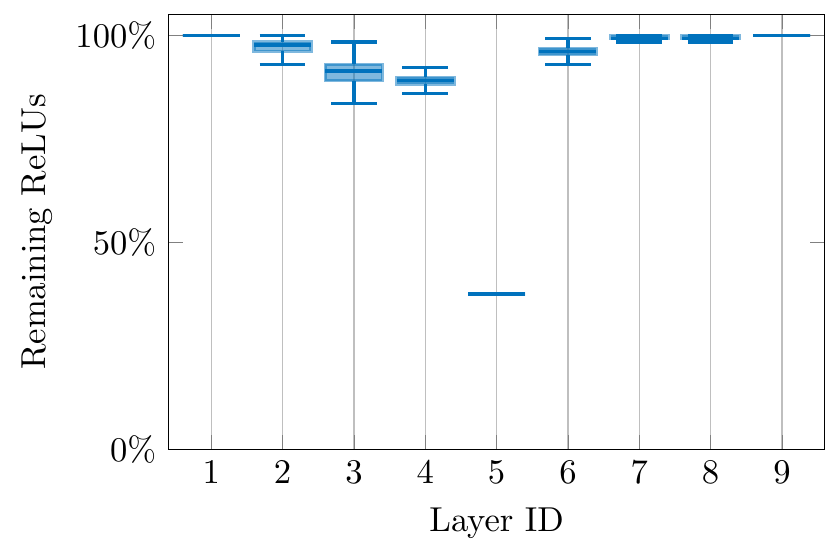}
        \caption{ToyADMOS}
        \label{fig:bar_anomaly}
    \end{subfigure}
    \begin{subfigure}[b]{0.66\textwidth}
    \centering
        \includegraphics[width=\textwidth]{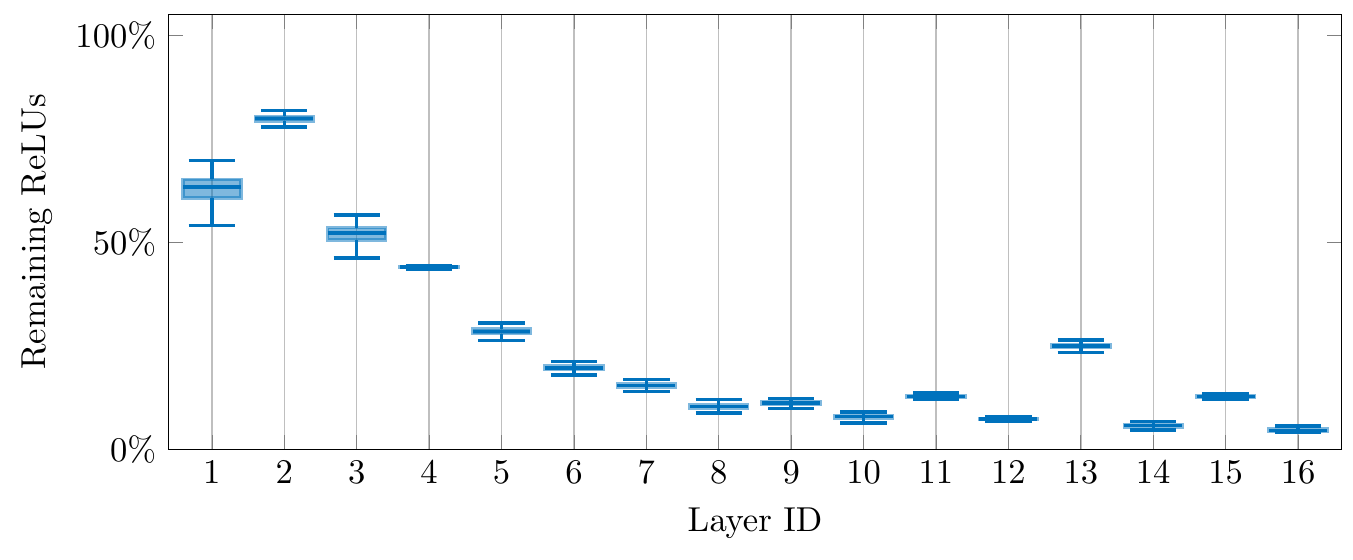}
        \caption{Visual Wake Words}
        \label{fig:bar_vww}
    \end{subfigure}
\caption{Percentage of ReLU activation functions remaining in the layers of our GINNACER abstractions, across multiple neural networks and input centroids.}
\label{fig:bar_compression}
\end{figure}

We report the results in Figure \ref{fig:bar_compression}, following the same box plot conventions of Figure \ref{fig:compression_time}. Note how the percentage of ReLUs remaining in the abstracted network varies greatly across different network architectures, layers and across input centroids. While it is difficult to discern general trends from our results, we speculate that ``bottleneck'' layers with high-dimensional and low-dimensional output yield fewer abstracted ReLUs. Examples of this phenomenon are the last layer of AcasXu networks (from 50 dimensions to 5), the last layer of the MNIST network (from 256 dimensions to 10) and the middle layer of the ToyADMOS autoencoder (from 128 dimensions to 8). Still, further analysis is required to understand whether this ``compressibility limit'' is a fundamental property of the network or just a consequence of our partition strategy.

Additionally, running GINNACER on the large layers of Visual Wake Words yields a higher reduction in the number of ReLU activation functions. We hypothesise that this is due to the sparsity of the weight matrices: each neuron in a convolutional neural layer is connected to a small subset of inputs only. As a result, the rank of the weight matrix is low, which allows for more aggressive partitioning of the neurons.

\section{Discussion}
\label{sec:discussion}

This paper proposes the GINNACER algorithm, which produces global abstractions of feedforward ReLU networks with local reconstruction guarantees. Here, we summarize the advantages of GINNACER over other global abstractions, outline our thoughts on its potential applications and propose a roadmap for future work.

\subsection{Advantages of GINNACER}
\label{sec:remarks}

The main novelty of GINNACER with respect to the other existing global abstractions methods \cite{Prabhakar2019,Elboher2020} is its center-exact reconstruction requirement. This requirement has two effects. On the one hand, it guides the abstraction procedure in introducing an uneven amount of over-approximation, which is concentrated away from the centroid. On the other hand, it constitutes a crucial piece of information to ``align'' all the abstracted layers, and make sure that the local pocket of zero over-approximation around the centroid is preserved throughout the whole neural network. The result of these design choices is an algorithm that is several orders of magnitude tighter than the other global abstraction methods \cite{Prabhakar2019,Elboher2020} \textit{even far away from the centroid}. In this respect, the main theoretical conclusion of our work is that ``aligning'' the over-approximation margin across different layers yields a global net benefit, not just a local one.

\subsection{Potential Applications}
\label{sec:applications}

Our contribution focuses on the problem of abstracting neural networks in isolation, i.e. without considering the potential applications of such abstractions. At the same time, we believe global neural network abstractions will soon find use in the following three areas.

First, neural network abstractions are routinely employed for verification purposes~\citep{Gehr2018,Liu2021}. In this respect, abstractions help neural network verifiers by simplifying the problem at hand while guaranteeing \textit{soundness}: if the abstracted network is safe, the original network is also safe. On the downside, verifying the abstracted network may lead to spurious counterexamples, as the verification problem is now \textit{incomplete}. For this reason, neural network verifiers often implement a so-called CounterExample-Guided Abstraction Refinement (CEGAR) iterative process~\citep{Katz2019,Bak2020}. Every time the verifier returns a spurious counterexample, the abstraction is refined to eliminate it, and the verification process is restarted. While most state-of-the-art verifiers use local abstraction techniques~\citep{Salman2019}, attempts at introducing global abstractions into a CEGAR loop have been made \citep{Elboher2020}. For the latter approach to become mainstream, we believe better global abstractions are needed. GINNACER is our contribution to such a research effort.

Second, neural network abstraction have been recently applied to control system problems~\cite{Sadeghzadeh2022,Newton2022,Adams2022}. In particular, abstractions with quadratic constraints are used to analyze the stability~\cite{Newton2022} and reachability~\cite{Sadeghzadeh2022} of closed-loop dynamic systems with neural controllers. Unfortunately, these quadratic abstractions are both conservative and local. We believe that GINNACER-style global abstractions can improve the design methodologies of control systems with neural networks by providing less conservative over-approximations.

Third, the connection between compression algorithms and neural network robustness has already caught the attention of safety researchers~\citep{Ye2019}. Notably, all three main compression techniques positively affect a neural network's robustness. That is, many forms of distillation~\citep{Fan2019}, pruning \citep{Ye2019,Jordao2021} and quantisation \citep{Lin2018b} reduce the density of adversarial perturbation attacks on the compressed network. Such a phenomenon has been observed empirically, but we still lack a solid theoretical explanation of it. We believe that research on global abstractions like GINNACER, which is fundamentally a compression technique with formal approximation guarantees, can provide the tools to build such a theoretical explanation.

\subsection{Future Work}
\label{sec:future_work}

While GINNACER is an essential step towards tackling these research challenges, several limitations remain. First, our algorithm is only applicable to ReLU networks so far. We believe extending it to general activation functions is possible, like the work of Paulsen and Wang~\citep{Paulsen2022} for local abstractions. Second, the centre-exact requirement of our algorithm limits the number of activation functions that can be removed, as opposed to the work of Prabhakar and Afzal~\citep{Prabhakar2019} and Elboher \textit{et al.}~\citep{Elboher2020}. We believe there is ample scope for improvement in this area. Finally, an ideal global abstraction would only produce a small over-approximation margin. While GINNACER improves the state-of-the-art by several orders of magnitude, it yields small margins only close to the centroid. We plan to address these limitations in our future work.

\section*{Acknowledgements}
\label{sec:acknowledge}

This work is funded by the EPSRC grant EP/T026995/1 entitled ``EnnCore: End-to-End Conceptual Guarding of Neural Architectures'' under \textit{Security for all in an AI enabled society}.

\bibliography{references_nnj}

\appendix

\section{Neuron Partition Strategy}
\label{sec:partioning}

The soundness of a global abstraction like the ones in~\cite{Prabhakar2019, Elboher2020} is agnostic to the choice of neurons we merge. In GINNACER, we add the extra requirement that the merged potentials are still negative at the centroid (see Definition~\ref{def:valid_subset}). At the same time, the \textit{quality} of the abstraction depends on the merging choices we make: intuitively, if we merge similar neurons together, we get a smaller over-approximation margin.

The existing literature approaches this problem in diametrically opposite ways. On the one hand, \cite{Elboher2020} settle on a greedy merging algorithm, that considers all possible pair of neurons and iteratively merges the one that causes the smallest error. Unfortunately, this greedy strategy becomes expensive as the number of neurons in a layer grows, as we show in our experiments in Section~\ref{sec:time_comparison}. On the other hand, \cite{Prabhakar2019} employs a stochastic merging algorithm, where a random pair of neurons is merged at each step. As our experiment in Section~\ref{sec:time_comparison} shows, this strategy is faster.

\begin{algorithm}[t]
\caption{Valid Neuron Partition}
\label{alg:neuron_partition}
\begin{algorithmic}[1]
\Require canonical layer weights $S^{\ell}W^{\ell}$ and biases $S^{\ell}b^{\ell}$, layer centroid $x^{\ell-1}_c$.
\Ensure valid neuron partition $D\equiv\{D_1,\dots,D_h\}$.
\State $D_i\gets\{i\},\forall i\in[1,n^{\ell}]$
\State $D\gets\{D_1,\dots,D_{n^{\ell}}\}$
\ForAll{$i\in[1,n^{\ell}-1]$}
    \ForAll{$j\in[i+1,n^{\ell}]$}
        \If{$D_i,D_j\in D$}
            \State $D_{ij}\gets D_i\cup D_j$
            \State $(V_{D_{ij}})_s\gets\max_{r\in D_{ij}}(S^{\ell}W^{\ell})_{rs},\forall s\in[1,n^{\ell-1}]$\hfill\Comment{Definition \ref{def:upper_bound}}
            \State $u_{D_{ij}}\gets\max_{r\in D_{ij}}(S^{\ell}b^{\ell})_r$\hfill\Comment{Definition \ref{def:upper_bound}}
            \State $\hat{r}_{D_{ij}}\gets(V_{D_{ij}})^Tx^{\ell-1}_c+u_{D_{ij}}$
            \If{$\hat{r}_{D_{ij}}\leq0$}\hfill\Comment{Definition \ref{def:valid_subset}}
                \State $D_i\gets D_{ij}$
                \State $D\gets D\setminus D_j$
            \EndIf
        \EndIf
    \EndFor
\EndFor
\end{algorithmic}
\end{algorithm}

Against this background, we propose to use the neuron partitioning strategy listed in Algorithm~\ref{alg:neuron_partition}. Similarly to the random merging algorithm in~\cite{Prabhakar2019}, our partitioning strategy does not try to find the best pair of neurons to merge. Instead, it iteratively merges any pair of neurons that result in a valid subset according to Definition~\ref{def:valid_subset}. More specifically, we start with a trivial partition where each neuron belongs to its own subset (Line 1). Then, for each remaining pair of subsets $D_i$ and $D_j$ (Line 5), we check whether merging them yields a negative potential at the centroid (Line 10). If so, we overwrite $D_i$ with its union with $D_j$ (Line 11), and remove $D_j$ from the partition (Line 12). The algorithm stops when no more subsets can be merged.

In Section~\ref{sec:algo_complexity} we derive the time complexity of Algorithm~\ref{alg:neuron_partition}. There, we show that it dominates the complexity of computing the GINNACER abstraction of any layer. These findings mirror our empirical results from Section~\ref{sec:time_comparison}.

\section{Abstraction and Over-Approximation Trade-Off}
\label{sec:tradeoff}

An abstraction is a compromise between a faithful representation of the original model, or a succinct one~\cite{Cousot1992}. GINNACER is always precise for inputs close to the centroid $x_c^{-1}$, but introduces a varying amount of over-approximation elsewhere. This over-approximation depends both on the distance to the centroid, and the number of ReLUs activation functions we remove from the original network.

In Section~\ref{sec:experiments}, we always show the behaviour of GINNACER in ``maximum compression'' mode. That is when we employ the partition strategy in~\ref{sec:partioning} to remove as many ReLUs as possible from the original network. However, GINNACER can be also used to remove a smaller number of ReLUs. Here, we illustrate the effect of such a strategy on the over-approximation margin.

\begin{figure}[t]
\centering
    \begin{subfigure}[b]{0.49\textwidth}
    \centering
        \includegraphics[width=\textwidth]{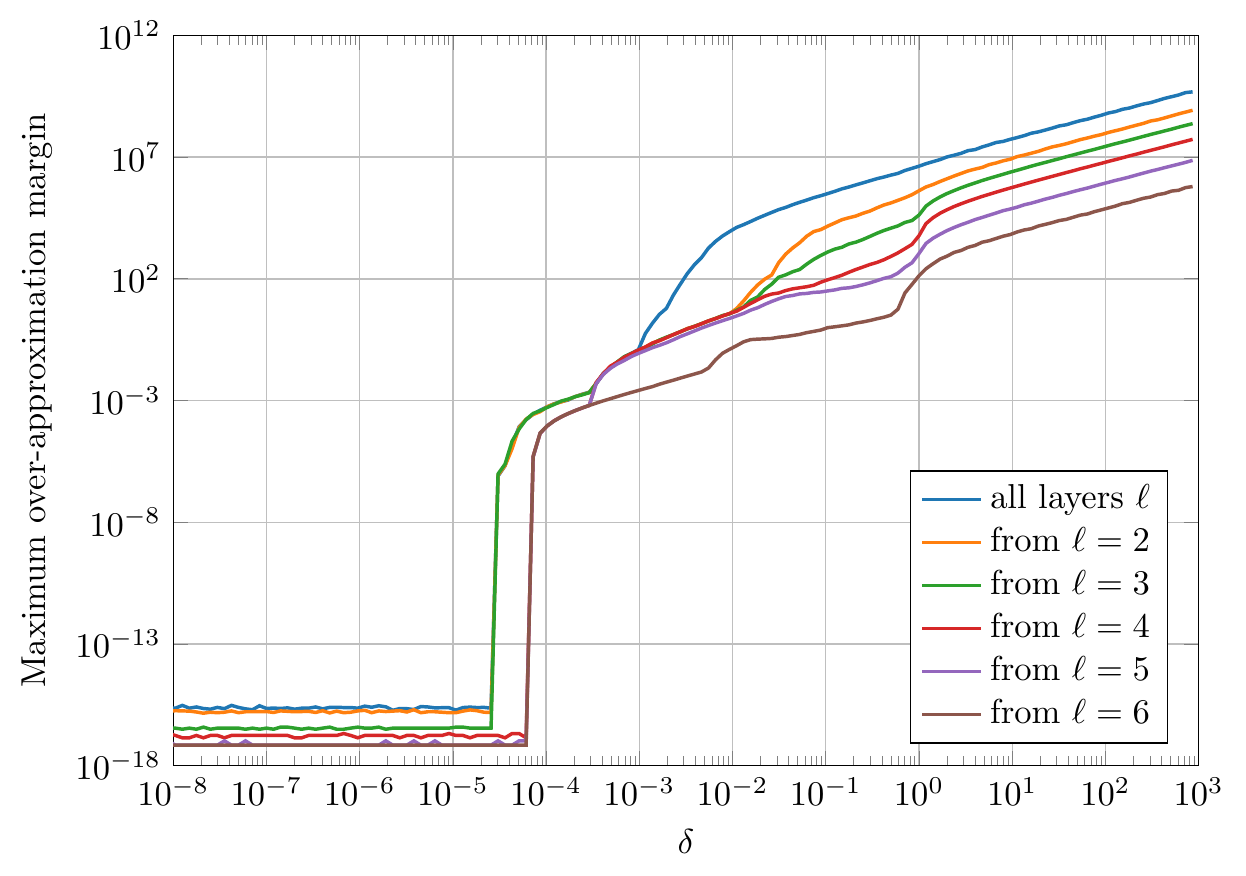}
        \caption{AcasXu (1-1)}
        \label{fig:acas_xu_tradeoff}
    \end{subfigure}
    \begin{subfigure}[b]{0.49\textwidth}
    \centering
        \includegraphics[width=\textwidth]{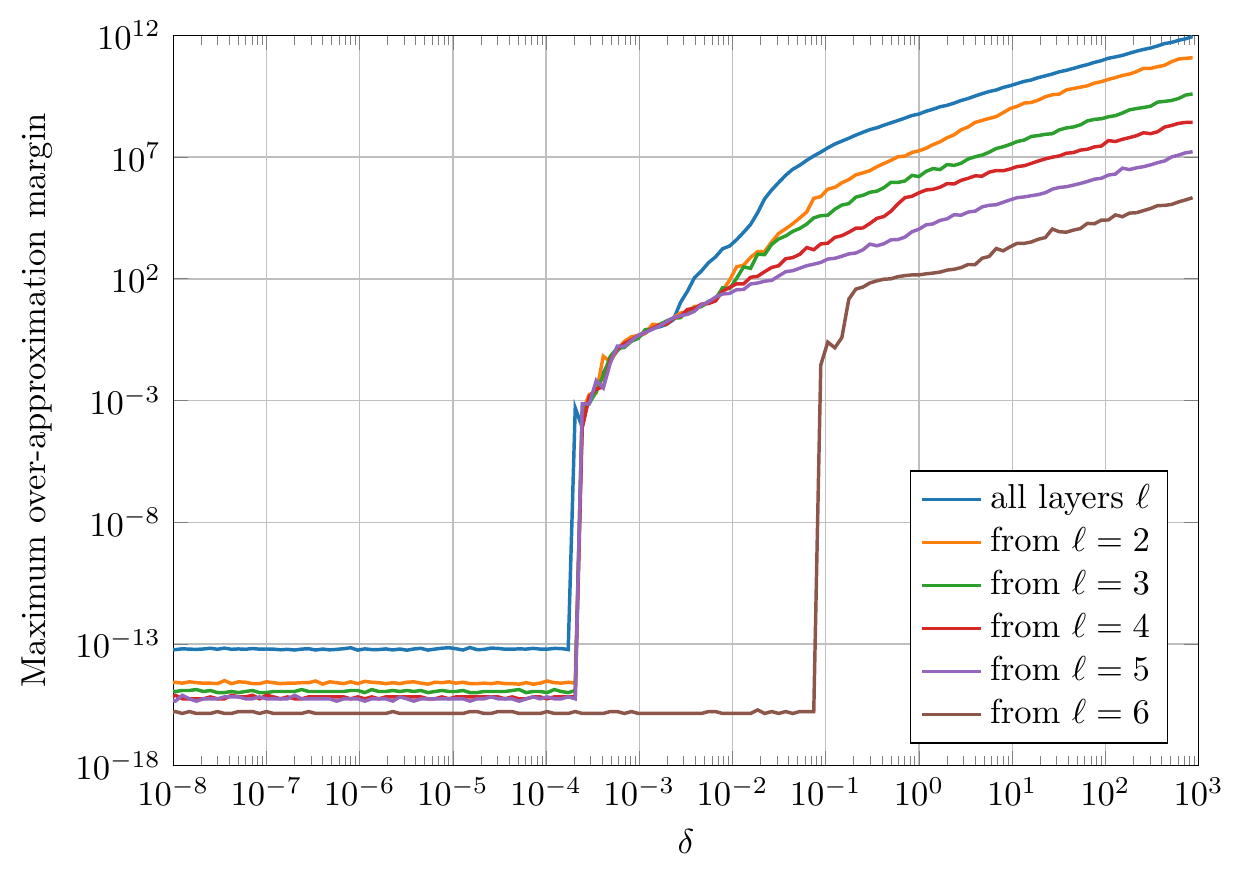}
        \caption{MNIST (fc)}
        \label{fig:mnist_fc_tradeoff}
    \end{subfigure}
    \begin{subfigure}[b]{0.49\textwidth}
    \centering
        \includegraphics[width=\textwidth]{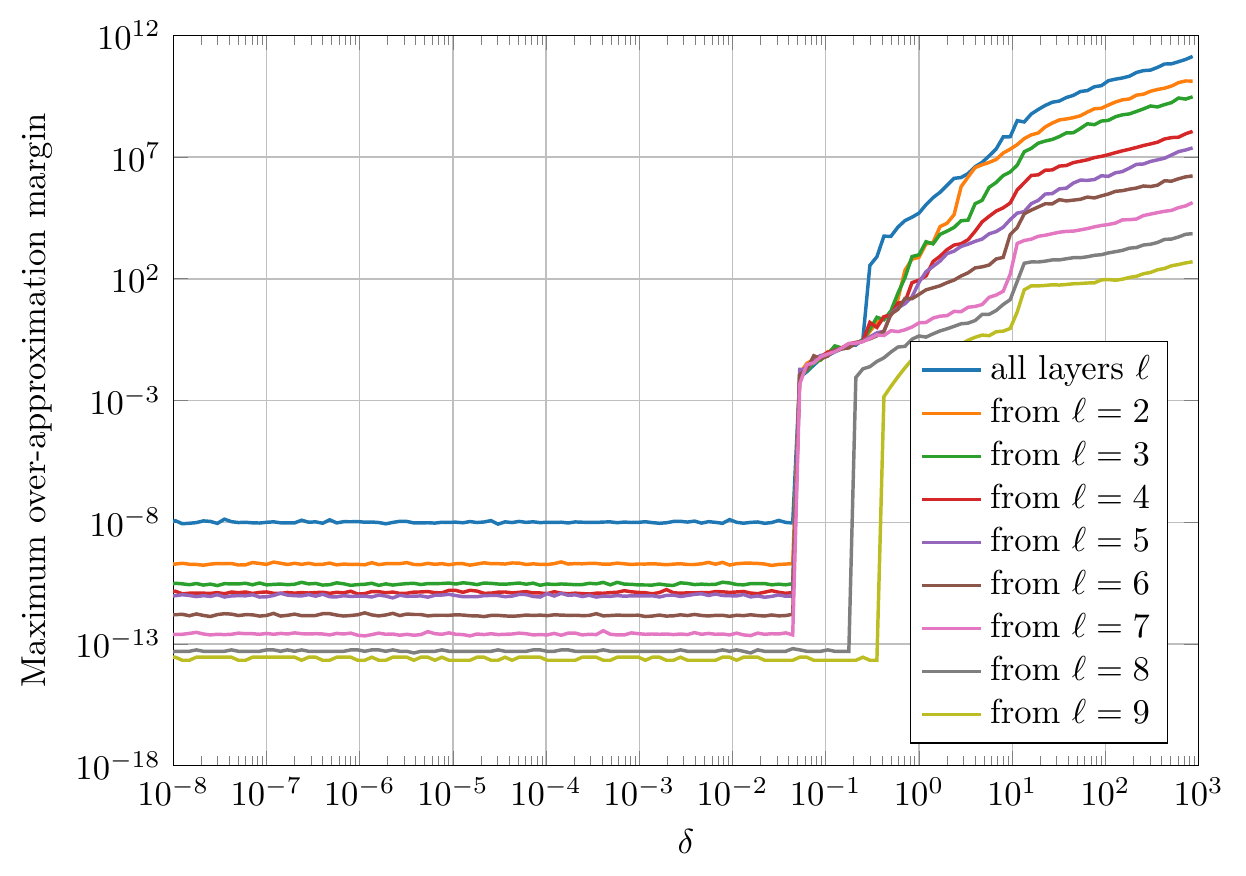}
        \caption{ToyADMOS}
        \label{fig:toy_admos_tradeoff}
    \end{subfigure}
    \begin{subfigure}[b]{0.49\textwidth}
    \centering
        \includegraphics[width=\textwidth]{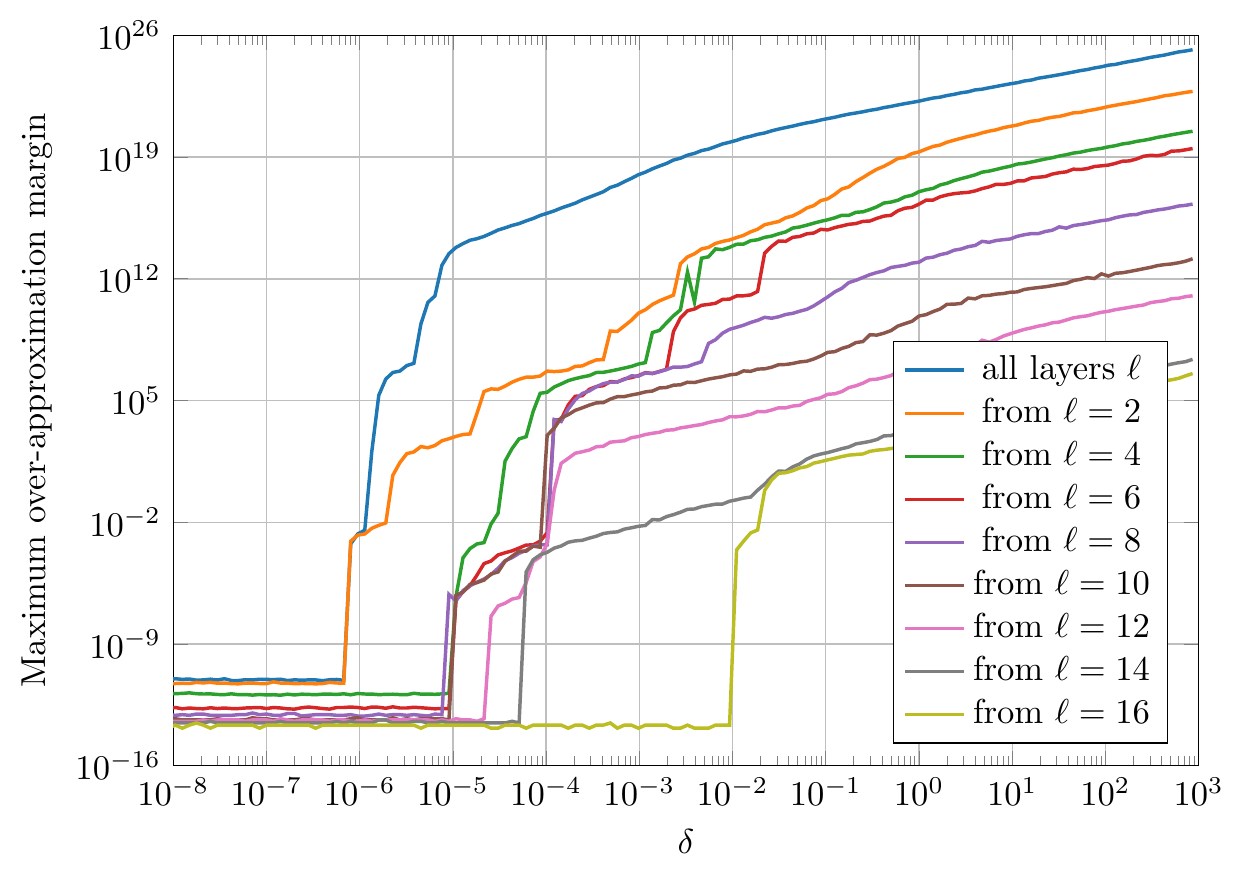}
        \caption{Visual Wake Words}
        \label{fig:vww_96_tradeoff}
    \end{subfigure}
\caption{Comparison between GINNACER abstractions with different settings. The ``all layers'' GINNACER abstraction is the same as the one presented in Figure \ref{fig:distance_comp}. The others abstract progressively fewer layers, starting from layer $\ell=2$ onward. Abstracting fewer layers yields smaller over-approximation margins.}
\label{fig:tradeoff_comp}
\end{figure}

Figure \ref{fig:tradeoff_comp} shows the maximum over-approximation margin of GINNACER for any input $x$ with distance $||x-x_c^{-1}||_{\infty}=\delta$ from the centroid $x_c^{-1}$. More specifically, the experiments were run under the same settings of our over-approximation comparison in Section~\ref{sec:empirical_comparison}. Accordingly, the line labeled ``all layers'' is identical to the maximum compression setting for GINNACER in Figure~\ref{fig:distance_comp}. In contrast, the other lines represent GINNACER abstractions where the first layer(s) are left untouched, and the ReLUs are removed from all the remaining layers.

The results in Figure~\ref{fig:tradeoff_comp} show a monotonic relationship between abstraction and over-approximation margin. This is expected since removing more ReLUs requires the upper and lower bounds in the abstraction to grow further apart. Note that a few non-monotonic data points exist in Figure~\ref{fig:tradeoff_comp}. We attribute them to the stochastic nature of our experiment, where the maximum over-approximation margin for a given $\delta$ is computed by independently sampling $10000$ random inputs on the surface of a $\delta$-sized hypercube.

\end{document}